\documentclass[twoside,11pt]{article}
\usepackage{jmlr2e}
\usepackage{amsmath}
\usepackage{color}
\usepackage{enumerate}
\usepackage{algorithm}

\jmlrheading{}{2014}{}{09/13; Revised 05/14}{}{Kean Ming Tan, Palma London, Karthik Mohan, Su-In Lee, Maryam Fazel, and Daniela Witten}

\ShortHeadings{Learning Graphical Models With Hubs}{Tan, London, Mohan, Lee, Fazel, and Witten}
\firstpageno{1}

\begin{document}

\title{Learning Graphical Models With Hubs}

\author{\name Kean Ming Tan \email keanming@uw.edu \\
       \addr Department of Biostatistics \\
       	     University of Washington \\
	           Seattle WA, 98195
       \AND
       \name Palma London \email palondon@uw.edu \\
       \name Karthik Mohan \email karna@uw.edu \\
       \addr Department of Electrical Engineering \\
       		University of Washington \\
			Seattle WA, 98195
       \AND
       \name Su-In Lee \email suinlee@cs.washington.edu \\
       \addr Department of Computer Science and  Engineering, Genome Sciences\\
       		University of Washington \\
			Seattle WA, 98195
       \AND
       \name Maryam Fazel \email mfazel@uw.edu \\
       \addr Department of Electrical Engineering \\
       		University of Washington \\
			Seattle WA, 98195
       \AND
       \name Daniela Witten \email dwitten@uw.edu \\
       \addr  Department of Biostatistics\\
       University of Washington\\
      Seattle, WA 98195}
      
\editor{Xiaotong Shen}
\maketitle
\begin{abstract}We consider the problem of learning a high-dimensional graphical model in which there are a few \emph{hub} nodes that are \emph{densely-connected} to many other nodes. Many authors have studied the use of an $\ell_1$ penalty in order to learn a sparse graph in the high-dimensional setting. 
 However, the  $\ell_1$ penalty implicitly assumes that each edge is equally likely and independent of all other edges.  We propose a general framework to accommodate more realistic networks with hub nodes, using a convex formulation that involves a row-column overlap norm penalty.  We apply this general framework to three widely-used probabilistic graphical models: the Gaussian graphical model,  the covariance graph model, and the binary Ising model.  An alternating direction method of multipliers algorithm is used to solve the corresponding convex optimization problems.  On synthetic data, we demonstrate that our proposed framework outperforms competitors that do not explicitly model hub nodes.  We illustrate our proposal on   a webpage data set and a gene expression data set. 
\end{abstract}

\section{Introduction}
\label{Sec:Introduction}
Graphical models are  used to model a wide variety of systems, such as gene regulatory networks and social interaction networks.  A graph consists of a set of $p$ nodes, each representing a variable, and a set of edges between pairs of nodes.  The presence of an edge between two nodes indicates a relationship between the two variables.  In this manuscript, we consider two types of graphs: conditional independence graphs and marginal independence graphs.  In a conditional independence graph,   an edge connects a pair of variables if and only if they are conditionally dependent---dependent conditional upon the other variables.    In a marginal independence graph, two nodes are joined by an edge if and only if they are marginally dependent---dependent without conditioning on the other variables.

In recent years, many authors have studied the problem of  learning a graphical model in the high-dimensional setting, in which the number of variables $p$ is larger than the number of observations $n$.  Let $\mathbf{X}$ be a $n\times p$ matrix, with rows  $\mathbf{x}_1,\ldots,\mathbf{x}_n$.  Throughout the rest of the text, we will focus on three specific types of graphical models:
\begin{enumerate}
\item A \emph{Gaussian graphical model}, where $\mathbf{x}_1,\ldots,\mathbf{x}_n \stackrel{\small \mathrm{i.i.d.}}\sim N(\mathbf{0},\mathbf{\Sigma})$.  In this setting, $(\mathbf{\Sigma}^{-1})_{jj'} = 0$ for some $j\ne j'$ if and only if the $j$th and $j'$th variables are conditionally independent \citep{MKB79}; therefore, the sparsity pattern of  $\mathbf{\Sigma}^{-1}$ determines the conditional independence graph.     

\item A \emph{Gaussian covariance graph model}, where $\mathbf{x}_1,\ldots,\mathbf{x}_n \stackrel{\small \mathrm{i.i.d.}}\sim N(\mathbf{0},\mathbf{\Sigma})$.  Then $\Sigma_{jj'}=0$ for some $j\ne j'$ if and only if the $j$th and $j'$th variables are marginally independent.  Therefore, the sparsity pattern of $\mathbf{\Sigma}$ determines the marginal independence graph. 

\item A \emph{binary Ising graphical model}, where $\mathbf{x}_1,\ldots,\mathbf{x}_n$ are i.i.d. with density function 
\[
p(\mathbf{x},\mathbf{\Theta}) = \frac{1}{Z(\mathbf{\Theta})}\exp  \left[ \sum_{j=1}^p \theta_{jj} x_j  +\sum_{1\le j < j' \le p} \theta_{jj'} x_j x_{j'}   \right],
\]
 $\mathbf{\Theta}$ is a $p\times p$ symmetric matrix, and $Z(\mathbf{\Theta})$ is the partition function, which ensures that the density sums to one.  Here, $\mathbf{x}$ is a binary vector, and $\theta_{jj'}=0$ if and only if the $j$th and $j'$th variables are conditionally independent.  
 The sparsity pattern of $\mathbf{\Theta}$ determines the conditional independence graph.  
\end{enumerate}

To construct an interpretable graph when $p>n$, many authors have proposed applying an $\ell_1$ penalty to the parameter encoding each edge, in order to encourage sparsity.
 For instance, such an approach is taken by 
 \citet{YuanLin07}, \citet{SparseInv}, \citet{Rothman08}, and \citet{YuanGlasso08} in the Gaussian graphical model;  \citet{ElKaroui2008}, \citet{BickelandLevina2008}, \citet{RothmanJASA09}, \citet{BienTibs11}, \citet{CaiLiu2011}, and \citet{Xueetal2012} in the covariance graph model; and 
\citet{LeeSIetal2007}, \citet{Hoefling2009}, and \citet{ravikumaretal2010} in the binary model.

 However, applying an $\ell_1$ penalty to each edge can be interpreted as placing an independent double-exponential prior on each edge.  Consequently, such an approach implicitly assumes that each edge is equally likely and independent of all other edges; this corresponds to an Erd\H{o}s-R\'{e}nyi graph in which most nodes have approximately the same number of edges  \citep{erdosrenyi}.  This is unrealistic in many real-world networks, in which we believe that certain nodes (which, unfortunately, are not known \emph{a priori}) have a lot more edges than other nodes.  An example is the network of webpages in the World Wide Web, where a relatively small  number of webpages are connected to many other webpages \citep{barabasi1999}.  A number of authors have shown that real-world networks are \emph{scale-free}, in the sense that the number of edges for each node follows a power-law distribution; examples include gene-regulatory networks, social networks, and networks of collaborations among scientists \citep[among others,][]{barabasi1999,barabasi2009,Liljerosetal2001,Jeongetal2001,Newman2000,li2005towards}.  More recently, \citet{Haoetal2012} have shown that certain genes, referred to as \emph{super hubs}, regulate hundreds of downstream genes in a gene regulatory network, resulting in far denser connections than are typically seen in a scale-free network.

In this paper, we refer to very densely-connected nodes, such as the ``super hubs" considered in \citet{Haoetal2012}, as \emph{hubs}.  When we refer to hubs, we have in mind nodes that are connected to a very substantial number of other nodes in the network---and in particular, we are referring to nodes that are much more densely-connected than even the most highly-connected node in a scale-free network. An example of a network containing hub nodes 
 is shown in Figure~\ref{Fig:ggmtoy}.

Here we propose a convex penalty function for estimating graphs containing hubs. Our formulation simultaneously identifies the hubs and estimates the entire graph.  The penalty function yields a convex optimization problem when combined with a convex loss function.  We consider the application of this hub penalty function in modeling Gaussian graphical models, covariance graph models, and binary Ising models.  Our formulation does not require that we know \emph{a priori} which nodes in the network are hubs.    

In related work, several authors have proposed methods to estimate a scale-free Gaussian graphical model \citep{QiangLiu2011,Defazio2012}.  However, those methods do not model hub nodes---the most highly-connected nodes that arise in a scale-free network are far less connected than the hubs that we consider in our formulation.  Under a different framework, some authors proposed a screening-based procedure to identify hub nodes in the context of Gaussian graphical models \citep{HeroandRajaratnam2012,firouzi2013local}.  Our proposal outperforms such approaches when hub nodes are present (see discussion in Section~\ref{subSec:other proposal}).  
   
   In Figure~\ref{Fig:ggmtoy}, the performance of our proposed approach is shown in a toy example in the context of a Gaussian graphical model. We see that when the true network contains hub nodes (Figure~\ref{Fig:ggmtoy}(a)), our proposed approach (Figure~\ref{Fig:ggmtoy}(b)) is much better able to recover the network than is the graphical lasso (Figure~\ref{Fig:ggmtoy}(c)), a well-studied approach that applies an $\ell_1$ penalty to each edge in the graph \citep{SparseInv}.
   \begin{figure}[htp]
\begin{center}
\includegraphics[scale=0.52]{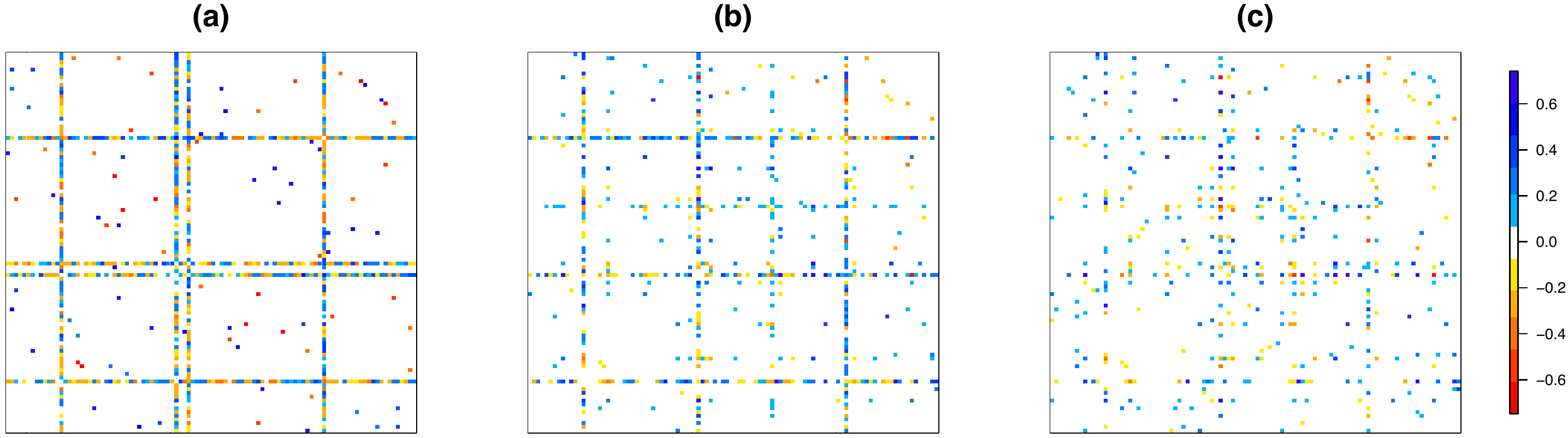}
 \end{center}
  \caption{(a): Heatmap of the inverse covariance matrix in a  toy example of a Gaussian graphical model with four hub nodes. White elements  are zero  and colored elements are non-zero in the inverse covariance matrix. Thus, colored elements correspond to edges in the graph. (b): Estimate from the \emph{hub graphical lasso}, proposed in this paper.  (c): Graphical lasso estimate. }
  \label{Fig:ggmtoy}
\end{figure}

We present the hub penalty function in Section~\ref{Sec:Penalty}.  We then apply it to the Gaussian graphical model, the covariance graph model, and the binary Ising model in Sections~\ref{Sec:GGM}, \ref{Sec:Covariance}, and \ref{Sec:Binary}, respectively.   In Section~\ref{Sec:realdata}, we apply our approach to   a webpage data set and a gene expression data set.     We close with a discussion in Section~\ref{Sec:Discussion}.

\section{The General Formulation}
\label{Sec:Penalty}
In this section, we present a general framework to accommodate network with hub nodes.
\subsection{The Hub Penalty Function}
 Let $\mathbf{X}$ be a $n \times p$ data matrix, $\mathbf{\Theta}$  a $p\times p$ symmetric matrix containing the parameters of interest, and $\ell(\mathbf{X},\mathbf{\Theta})$ a loss function (assumed to be convex in $\bf \Theta$). 
In order to obtain a sparse and interpretable graph estimate, many authors have considered the problem
\begin{eqnarray}
\label{Eq:l1general}
 \underset{{\mathbf{\Theta}\in \mathcal{S}}}{\mathrm{minimize}}
& & \left\{ \ell(\mathbf{X},\mathbf{\Theta})  + \lambda \| \mathbf{\Theta} - \mathrm{diag}(\mathbf{\Theta})\|_1 \right \},
\end{eqnarray}
where $\lambda$ is a non-negative tuning parameter, $\mathcal{S}$ is some set depending on the loss function, and $\|\cdot \|_1$ is the sum of the absolute values of the matrix elements.  For instance, in the case of a Gaussian graphical model, we could take $\ell(\mathbf{X},\mathbf{\Theta}) = -\log \det {\bf \Theta} + \mbox{trace}({\bf S} {\bf \Theta})$, the negative log-likelihood of the data, where $\bf S$ is the empirical covariance matrix and $\mathcal{S}$ is the set of $p\times p$ positive definite matrices. The solution to (\ref{Eq:l1general}) can then be interpreted as an estimate of the inverse covariance matrix. The $\ell_1$ penalty in (\ref{Eq:l1general}) encourages zeros in the solution. But it  typically does not yield an estimate that contains hubs.   

\begin{figure}[htp]
\begin{center}
\includegraphics[scale=0.625]{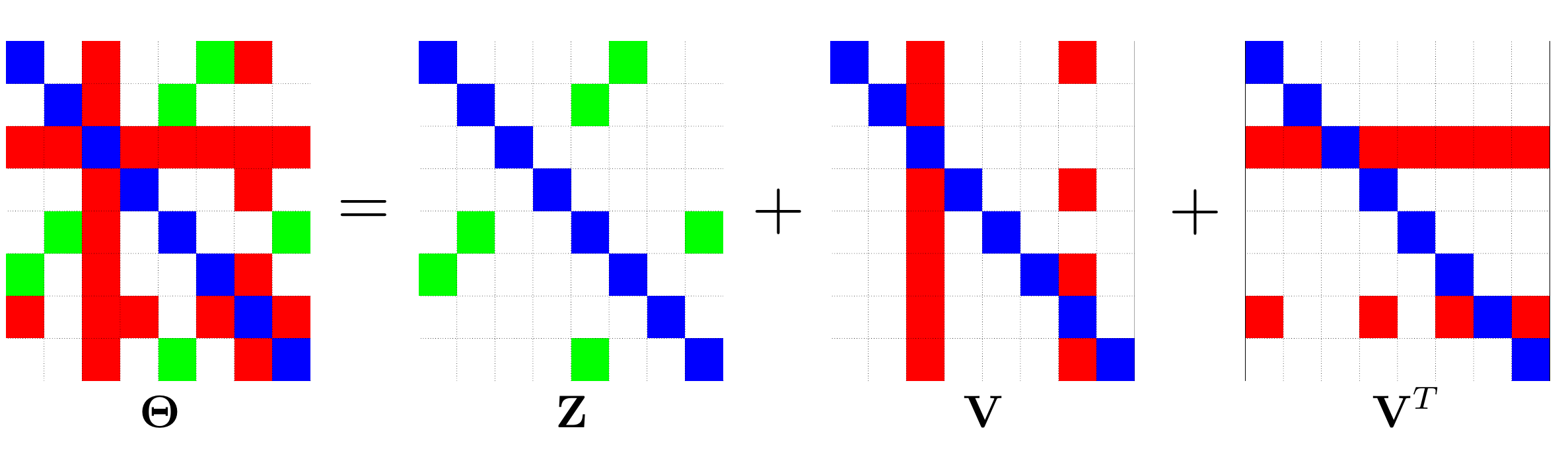}
\end{center}
\caption{Decomposition of a symmetric matrix $\mathbf{\Theta}$ into $\mathbf{Z+V+V}^T$, where $\mathbf{Z}$ is sparse, and most columns of $\mathbf{V}$ are entirely zero.  Blue, white, green, and red elements are  diagonal, zero, non-zero in $\mathbf{Z}$, and non-zero due to two hubs in $\mathbf{V}$, respectively.}
\label{Fig:formulation}
\end{figure}

In order to explicitly model hub nodes in a graph, we wish to replace the $\ell_1$ penalty in (\ref{Eq:l1general}) with a convex penalty that encourages a solution that can be decomposed as $\mathbf{Z}+\mathbf{V}+\mathbf{V}^T$, where $\mathbf{Z}$ is a sparse symmetric  matrix, and $\mathbf{V}$ is a matrix whose columns are either entirely zero or almost entirely non-zero (see Figure \ref{Fig:formulation}).  The sparse elements of $\mathbf{Z}$ represent edges between non-hub nodes, and the non-zero columns of $\mathbf{V}$   correspond to hub nodes. We achieve this goal via the \emph{hub penalty function}, which takes the form 
\begin{equation}
\label{Eq:hubpenalty} \footnotesize
 \text{P}(\mathbf{\Theta}) = \underset{\mathbf{V},\mathbf{Z}: \;\;  \mathbf{\Theta} = \mathbf{V}+\mathbf{V}^T+\mathbf{Z}} {\text{min}} 
  \left \{ \lambda_1 \| \mathbf{Z} - \text{diag}(\mathbf{Z})\|_1 +\lambda_2  \| \mathbf{V} - \text{diag}(\mathbf{V})\|_1 
 +\lambda_3 \sum_{j=1}^p \| (\mathbf{V}  - \text{diag}(\mathbf{V}))_j \|_q \right\}.
\end{equation}
Here $\lambda_1, \lambda_2$, and $\lambda_3$ are nonnegative tuning parameters. Sparsity in $\mathbf{Z}$ is encouraged via the $\ell_1$ penalty on its off-diagonal elements, and is controlled by the value of $\lambda_1$. The  $\ell_1$ and $\ell_1/ \ell_q$ norms on the columns of $\bf V$  induce group sparsity  when $q=2$ \citep{grouplasso,sparsegrouplasso}; $\lambda_3$ controls the selection of hub nodes, and $\lambda_2$ controls the sparsity of each hub node's connections to other nodes. 
The convex penalty (\ref{Eq:hubpenalty}) can be combined with $\ell({\bf X}, {\bf \Theta})$ to yield the convex optimization problem
\begin{eqnarray}
\label{Eq:general}
 \underset{{\mathbf{\Theta}\in \mathcal{S}, \mathbf{V, Z}}} {\text{minimize}}
& & \Bigg\{ \ell(\mathbf{X},\mathbf{\Theta})  + \lambda_1 \| \mathbf{Z} - \text{diag}(\mathbf{Z})\|_1 +\lambda_2  \| \mathbf{V} - \text{diag}(\mathbf{V})\|_1 \nonumber \\
&&+ \lambda_3 \sum_{j=1}^p \| (\mathbf{V}  - \text{diag}(\mathbf{V}))_j \|_q \Bigg\} \;\; \text{subject to} \;\;   \mathbf{\Theta} = \mathbf{V}+\mathbf{V}^T+\mathbf{Z},
\end{eqnarray}
where the set $\mathcal{S}$ depends on the loss function $\ell(\mathbf{X},\mathbf{\Theta})$.

Note that when $\lambda_2\rightarrow \infty$ or $\lambda_3 \rightarrow \infty$, then (\ref{Eq:general}) reduces to  (\ref{Eq:l1general}).   In this paper, we take $q=2$, which leads to estimation of a network containing dense hub nodes.  Other values of $q$  such as $q=\infty$ are also possible 
   \citep[see, e.g.,][]{Mohanetal2013}. We note that the hub penalty function is closely related to recent work on overlapping group lasso penalties in the context of learning multiple sparse precision matrices \citep{Mohanetal2013}.

\subsection{Algorithm}
In order to solve  (\ref{Eq:general}) with $q=2$, we use an \emph{alternating direction method of multipliers} (ADMM) algorithm \citep[see, e.g.,][]{EcksteinADMM92,BoydADMM,ADMMconvergence}.  ADMM is an attractive algorithm for this problem, as it allows us to decouple some of the terms in (\ref{Eq:general}) that are difficult to optimize jointly.  In order to develop an ADMM algorithm for (\ref{Eq:general}) with guaranteed convergence, we reformulate it as 
 a consensus problem, as in \citet{Maetal2013ADMM}. The convergence of the algorithm to the optimal solution follows from classical results \citep[see, e.g., the review papers][]{BoydADMM,ADMMconvergence}.  

\begin{algorithm}[htp]
\small
\caption{ADMM Algorithm for Solving (\ref{Eq:general}).}
\label{Alg:general}
\begin{enumerate}
\item  \textbf{Initialize} the parameters: 
\begin{enumerate}
\item primal variables $\mathbf{\Theta,V,Z}, \tilde{\mathbf{\Theta}},\tilde{\mathbf{V}}$, and $\tilde{\mathbf{Z}}$ to the $p \times p$ identity matrix.
\item dual variables $\mathbf{W}_1,\mathbf{W}_2$, and $\mathbf{W}_3$ to the $p \times p$ zero matrix.
\item  constants $\rho>0$ and $\tau>0$.\\
 
\end{enumerate}
 
\item  \textbf{Iterate} until the stopping criterion $\frac{\| {\mathbf{\Theta}}_{t}- {\mathbf{\Theta}}_{t-1} \|_F^2}{\| {\mathbf{\Theta}}_{t-1}\|_F^2} \le \tau$ is met, where ${\bf \Theta}_t$ is the value of $\bf \Theta$ obtained at the $t$th iteration:

\begin{enumerate}
\item Update ${\bf \Theta}, {\bf V}, {\bf Z}$:
\begin{enumerate}
\item $\mathbf{\Theta}= \underset{\mathbf{\Theta}\in \mathcal{S}}{\arg \min} \left \{  \ell(\mathbf{X},\mathbf{\Theta}) + \frac{\rho}{2} \|\mathbf{\Theta}-\tilde{\mathbf{\Theta}}+\mathbf{W}_{1}\|_F^2 \right \}$.  
\item $\mathbf{Z}= S(\tilde{\bf Z} - \mathbf{W}_3, \frac{\lambda_1}{\rho})$,
diag$\mathbf{(Z)}= \text{diag}(\tilde{\mathbf{Z}}-\mathbf{W}_3)$. Here $S$ denotes the soft-thresholding operator, applied element-wise to a matrix:  $S(A_{ij},b) = \text{sign}(A_{ij}) \max( |A_{ij}|-b, 0)$.

\item $\mathbf{C}= \tilde{\mathbf{V}}-\mathbf{W}_2-\text{diag}(\tilde{\mathbf{V}}-\mathbf{W}_2)$.
\item  $\mathbf{V}_j= \max \left(1-\frac{\lambda_3}{\rho \|S(\mathbf{C}_j,\lambda_2/\rho) \|_2}, 0 \right) \cdot S(\mathbf{C}_j, \lambda_2/\rho)$ for $j=1, \ldots, p$. 
\item diag$(\mathbf{V}) = \text{diag}(\tilde{\mathbf{V}}-\mathbf{W}_2)$.
\end{enumerate}

\item Update   $\tilde{\bf \Theta}, \tilde{\bf V}, \tilde{\bf Z}$:
\begin{enumerate}
\item $\mathbf{\Gamma} = \frac{\rho}{6}\left[ ({\mathbf{\Theta+W}_1}) - (\mathbf{V+W}_2) -(\mathbf{V+W}_2)^T - (\mathbf{Z+W}_3)   \right]$.

\item $\tilde{\mathbf{\Theta}}= {\bf \Theta + W}_1 - \frac{1}{\rho}\mathbf{\Gamma}$;   \;  \;\;  iii. $\tilde{\mathbf{V}} = \frac{1}{\rho} (\mathbf{\Gamma+\Gamma}^T)  +\mathbf{V+W}_2$;   \;  \;\;  iv. $\tilde{\mathbf{Z}} = \frac{1}{\rho} \mathbf{\Gamma} + {\bf Z+W}_3$.

\end{enumerate}

\item Update  $\mathbf{W}_1, {\bf W}_2, {\bf W}_3$:
\begin{enumerate}
\item $\mathbf{W}_1 = \mathbf{W}_1 + \mathbf{\Theta}-\tilde{\bf \Theta}$;   \;  \;\;  ii. $\mathbf{W}_2 = \mathbf{W}_2 + \mathbf{V}-\tilde{\bf V}$;   \;  \;\;  iii. $\mathbf{W}_3 = \mathbf{W}_3 + \mathbf{Z}-\tilde{\bf Z}$.
\end{enumerate}

\end{enumerate}
\end{enumerate}
\end{algorithm}

 In greater detail, we let $\mathbf{B}=(\mathbf{\Theta},\mathbf{V},\mathbf{Z})$, $\tilde{\mathbf{B}}=(\tilde{\mathbf{\Theta}},\tilde{\mathbf{V}},\tilde{\mathbf{Z}})$, 
\[
f(\mathbf{B}) = \ell(\mathbf{X,\Theta}) + \lambda_1 \|\mathbf{Z}-\text{diag}(\mathbf{Z}) \|_1+ \lambda_2 \|\mathbf{V}-\text{diag}(\mathbf{V}) \|_1+ \lambda_3 \sum_{j=1}^p \| (\mathbf{V}-\text{diag}(\mathbf{V})) \|_2,
\]
and
\[
g(\tilde{\mathbf{B}})=\begin{cases} 0 & \text{if } \tilde{\mathbf{\Theta}} = \tilde{\mathbf{V}}+\tilde{\mathbf{V}}^T + \tilde{\mathbf{Z}}  \\ \infty & \text{otherwise}.\end{cases}
\]
\noindent Then, we can rewrite (\ref{Eq:general}) as 
\begin{equation}
\label{Eq:reformulate}
\underset{\mathbf{B},\tilde{\mathbf{B}}}{\text{minimize  }} \left\{f(\mathbf{B})+g(\tilde{\mathbf{B}})\right\} \qquad \text{subject to } \mathbf{B}=\tilde{\mathbf{B}}.
\end{equation}

\noindent  The scaled augmented Lagrangian for (\ref{Eq:reformulate}) takes the form 
 \begin{equation*}
\begin{split}
L(\mathbf{B},\tilde{\mathbf{B}},\mathbf{W}) &= \ell(\mathbf{X},\mathbf{\Theta})    + \lambda_1 \| \mathbf{Z} - \text{diag}(\mathbf{Z})\|_1+ \lambda_2  \| \mathbf{V} - \text{diag}(\mathbf{V})\|_1 \\
&+\lambda_3 \sum_{j=1}^p \| (\mathbf{V}  - \text{diag}(\mathbf{V}))_j \|_2 +g(\tilde{\mathbf{B}})  +\frac{\rho}{2}\|\mathbf{B}-\tilde{\mathbf{B}}+\mathbf{W}  \|^2_F,\\
\end{split}
\end{equation*}
where $\mathbf{B}$ and $\tilde{\mathbf{B}}$ are the primal variables, and ${\bf W}=({\bf W}_1, {\bf W}_2, {\bf W}_3 )$ is the dual variable.  
 Note that the scaled augmented Lagrangian can be derived from the usual Lagrangian by adding a quadratic term and completing the square \citep{BoydADMM}.  
 
  A general algorithm for solving
 (\ref{Eq:general})   is provided in Algorithm \ref{Alg:general}.  The derivation is in Appendix A.  Note that only the update for $\mathbf{\Theta}$ (Step 2(a)i) depends on the form of the convex loss function $\ell(\mathbf{X},\mathbf{\Theta})$. In the following sections, we consider special cases of (\ref{Eq:general}) that lead to estimation of  Gaussian graphical models, covariance graph models, and binary networks with hub nodes.          

\section{The Hub Graphical Lasso}
\label{Sec:GGM}
 Assume that $\mathbf{x}_1,\ldots,\mathbf{x}_n \stackrel{\small\mathrm{i.i.d.}}\sim N(\mathbf{0},\mathbf{\Sigma})$.   The well-known \emph{graphical lasso} problem   \citep[see, e.g.,][]{SparseInv} takes the form of (\ref{Eq:l1general}) with $\ell(\mathbf{X},\mathbf{\Theta}) = -\log \det {\bf \Theta} + \mbox{trace}({\bf S} {\bf \Theta})$, and $\bf S$ the empirical covariance matrix of $\mathbf{X}$:
 \begin{equation}
\label{Eq:l1penalizeggm}
\underset{\mathbf{\Theta}\in \mathcal{S}}{\text{minimize}} \quad \left\{-\log \det  \mathbf{\Theta} + \text{trace}(\mathbf{S\Theta}) + \lambda \sum_{j\ne j'} |{\Theta}_{jj'}| \right\},
\end{equation}
where $\mathcal{S} = \{\mathbf{\Theta}: \mathbf{\Theta} \succ 0 \text{ and } \mathbf{\Theta}=\mathbf{\Theta}^T   \}$.    The   solution to this optimization problem serves as an estimate for ${\bf \Sigma}^{-1}$. We now use the hub penalty function to extend the graphical lasso in order to accommodate hub nodes.

\subsection{Formulation and Algorithm}
\label{GGM:formulation}
We propose the \emph{hub graphical lasso} (HGL) optimization problem, which  takes the form
\begin{eqnarray}
\label{Eq:ggmhub}
 \underset{{\mathbf{\Theta}}\in \mathcal{S}} {\text{minimize}}
& & \left \{ -\log \det  \mathbf{\Theta} + \text{trace}(\mathbf{S\Theta})  + \text{P}(\mathbf{\Theta}) \right\}.
\end{eqnarray}
Again, $\mathcal{S} = \{\mathbf{\Theta}:\mathbf{\Theta} \succ 0 \text{ and } \mathbf{\Theta}=\mathbf{\Theta}^T \}$.
\noindent It encourages a solution that contains hub nodes, as well as edges that connect non-hubs (Figure~\ref{Fig:ggmtoy}).
Problem (\ref{Eq:ggmhub}) can be solved using Algorithm \ref{Alg:general}.
The update for $\mathbf{\Theta}$ in Algorithm \ref{Alg:general} (Step 2(a)i) can be derived by minimizing 
\begin{equation}
-\log \det  \mathbf{\Theta} + \text{trace}(\mathbf{S\Theta}) + \frac{\rho}{2}\| \mathbf{\Theta} - \tilde{\mathbf{\Theta}} + \mathbf{W}_1 \|_F^2
\end{equation}
with respect to $\mathbf{\Theta}$ (note that the constraint $\mathbf{\Theta} \in \mathcal{S}$ in (\ref{Eq:ggmhub}) is treated as an implicit constraint, due to the domain of definition of the $\log \det$ function).  This can be shown to have the solution
\[
{\mathbf{\Theta}}=\frac{1}{2}\mathbf{U}\left( \mathbf{D}+ \sqrt{\mathbf{D}^2+\frac{4}{\rho}\mathbf{I}} \right) \mathbf{U}^T,
\]
where $\mathbf{UDU}^T$ denotes the eigen-decomposition of $\tilde{\mathbf{\Theta}}-\mathbf{W}_1 -\frac{1}{\rho}\mathbf{S}$.

The complexity of the ADMM algorithm for HGL is $O(p^3)$ per iteration; this is the complexity of the eigen-decomposition for updating $\mathbf{\Theta}$.   We now briefly compare the computational time for the ADMM algorithm for solving (\ref{Eq:ggmhub}) to that of an interior point method (using the solver \verb=Sedumi= called from \verb=cvx=).  On a 1.86 GHz Intel Core 2 Duo machine, the interior point method  takes $\sim 3$ minutes, while ADMM takes only 1 second, on a data set with $p=30$.  We present a more extensive run time study for the ADMM algorithm for HGL in Appendix E.  

\subsection{Conditions for HGL Solution to be Block Diagonal}
\label{GGM:blockdiagonal}
In order to reduce computations for solving the HGL problem, we now present a necessary condition and a sufficient condition for the HGL solution to be block diagonal, subject to some permutation of the rows and columns.  The conditions depend only on the tuning parameters $\lambda_1$, $\lambda_2$, and $\lambda_3$.   These conditions build upon similar results in the context of Gaussian graphical models from the recent literature \citep[see, e.g.,][]{WittenFriedman11,MazumderHastie11,yang2012fused,Danaheretal2012,Mohanetal2013}.
Let $C_1, C_2, \ldots, C_K$ denote a partition of the $p$ features.
\begin{theorem}
 A sufficient condition for the  HGL solution to be block diagonal with blocks given by $C_1,C_2,\ldots,C_K$  is that $\min \left\{\lambda_1, \frac{\lambda_2}{2} \right\} > |S_{jj'}|$  for all $j\in C_k, j' \in C_{k'}, k\ne k'$.
\label{Theorem:BD}
\end{theorem}

\begin{theorem}
 A necessary condition for the  HGL solution to be block diagonal with blocks given by $C_1,C_2,\ldots,C_K$  is that $\min\left\{\lambda_1, \frac{\lambda_2+\lambda_3}{2}\right\} > |S_{jj'}|$  for all $j\in C_k, j' \in C_{k'}, k\ne k'$.
\label{Theorem:BD2}
\end{theorem}

Theorem~\ref{Theorem:BD} implies that one can screen the empirical covariance matrix $\mathbf{S}$ to check if the HGL solution is block diagonal \citep[using standard algorithms for identifying the connected components of an undirected graph; see, e.g.,][]{tarjan1972depth}.  Suppose that the HGL solution is block diagonal with $K$ blocks, containing $p_1,\ldots,p_K$ features, and $\sum_{k=1}^K p_k = p$.  Then, one can simply solve the HGL problem on the features within each block separately.  Recall that the bottleneck of the HGL algorithm is the eigen-decomposition for updating $\mathbf{\Theta}$.  The block diagonal condition leads to massive computational speed-ups for implementing the HGL algorithm: instead of computing an eigen-decomposition for a $p\times p$ matrix in each iteration of the HGL algorithm, we compute the eigen-decomposition of $K$ matrices of dimensions $p_1\times p_1,\ldots,p_K \times p_K$.  The computational complexity per-iteration is reduced from $O(p^3)$ to $\sum_{k=1}^K O(p_k^3)$.   

We illustrate the reduction in computational time due to these results in an example with $p=500$.  Without exploiting Theorem~\ref{Theorem:BD}, the ADMM algorithm for HGL  (with a particular value of $\lambda$) takes 159 seconds; in contrast, it takes only 22 seconds when Theorem~\ref{Theorem:BD} is applied.  The estimated precision matrix has 107 connected components, the largest of which contains 212 nodes.

\subsection{Some Properties of  HGL}
\label{GGM:properties}

We now present several properties of the HGL optimization problem (\ref{Eq:ggmhub}), which can be used to provide guidance on the suitable range for the tuning parameters $\lambda_1$, $\lambda_2$, and $\lambda_3$.  In what follows, ${\bf Z}^*$ and ${\bf V}^*$ denote the optimal solutions for $\bf Z$ and $\bf V$ in (\ref{Eq:ggmhub}). Let $\frac{1}{s}+\frac{1}{q} =1$ (recall that $q$ appears in (\ref{Eq:hubpenalty})).

\begin{lemma}
\label{Lemma:DiagonalZ}
A sufficient condition for $\mathbf{Z}^*$ to be a diagonal matrix is that $\lambda_1 > \frac{\lambda_2+\lambda_3}{2}$.
\end{lemma}

\begin{lemma}
\label{Lemma:DiagonalV}
A sufficient condition for $\mathbf{V}^*$ to be a diagonal matrix is that $\lambda_1 < \frac{\lambda_2}{2}+\frac{\lambda_3}{2(p-1)^{1/s}}$.
\end{lemma}

\begin{corollary}
A necessary condition for both $\mathbf{V}^*$ and $\bf{Z}^*$ to be non-diagonal matrices is that $\frac{\lambda_2}{2}+\frac{\lambda_3}{2(p-1)^{1/s}} \le \lambda_1 \le \frac{\lambda_2+\lambda_3}{2}$.
\end{corollary}
Furthermore, (\ref{Eq:ggmhub}) reduces to the graphical lasso  problem (\ref{Eq:l1penalizeggm}) under a simple condition.

\begin{lemma}
\label{lemma3}
If $q=1$, then (\ref{Eq:ggmhub}) reduces to (\ref{Eq:l1penalizeggm}) with tuning parameter $\min \left \{\lambda_1, \frac{\lambda_2+\lambda_3}{2} \right\}$.
\end{lemma}
Note also that when $\lambda_2 \rightarrow \infty$ or $\lambda_3 \rightarrow \infty$,  (\ref{Eq:ggmhub}) reduces to (\ref{Eq:l1penalizeggm}) with tuning parameter $\lambda_1$.  However, throughout the rest of  this paper, we assume that $q=2$, and $\lambda_2$ and $\lambda_3$ are finite.

The solution $\hat{\mathbf{\Theta}}$  of (\ref{Eq:ggmhub}) is unique, since (\ref{Eq:ggmhub}) is a strictly convex problem.  We now consider the question of whether the decomposition $\hat{\mathbf{\Theta}}= \hat{\mathbf{V}} +     \hat{\mathbf{V}}^T + \hat{\mathbf{Z}}$ is unique.  We see that the decomposition is unique in a certain regime of the tuning parameters.  For instance, according to Lemma~\ref{Lemma:DiagonalZ}, when $\lambda_1 > \frac{\lambda_2+\lambda_3}{2}$, $\hat{\mathbf{Z}}$ is a diagonal matrix and hence 
$\hat{\mathbf{V}}$ is unique.        Similarly, according to Lemma~\ref{Lemma:DiagonalV}, when $\lambda_1 < \frac{\lambda_2}{2} + \frac{\lambda_3}{2(p-1)^{1/s}}$, $\hat{\mathbf{V}}$ is a diagonal matrix and hence $\hat{\mathbf{Z}}$ is unique.  
Studying more general conditions on $\mathbf{S}$ and on $\lambda_1$, $\lambda_2$, and $\lambda_3$ such that the decomposition is guaranteed to be unique is a challenging problem and is outside of the scope of this paper.

\subsection{Tuning Parameter Selection}
\label{Sec:tuning parameter}
In this section, we propose a \emph{Bayesian information criterion} (BIC)-type quantity for tuning parameter selection in (\ref{Eq:ggmhub}).  Recall from Section~\ref{Sec:Penalty} that the hub penalty function (\ref{Eq:hubpenalty}) decomposes the parameter of interest into the sum of three matrices, $\mathbf{\Theta} = \mathbf{Z+V+V}^T$,  and places an $\ell_1$ penalty on $\mathbf{Z}$, and an $\ell_1/\ell_2$ penalty on $\mathbf{V}$.

For the graphical lasso problem in (\ref{Eq:l1penalizeggm}), many authors have proposed to select the tuning parameter $\lambda$ such that  $\hat{\mathbf{\Theta}}$ minimizes the following quantity:
\[
-n \cdot \log \det (\hat{\mathbf{\Theta}}) + n\cdot  \text{trace}(\mathbf{S}\hat{\mathbf{\Theta}}) + \log (n) \cdot |\hat{\mathbf{\Theta}}|,
\]
where $|\hat{\mathbf{\Theta}}|$ is the cardinality of $\hat{\mathbf{\Theta}}$, that is, the number of unique non-zeros in $\hat{\mathbf{\Theta}}$ \citep[see, e.g.,][]{YuanLin07}.\footnote{The term $\log(n) \cdot |\hat{\mathbf{\Theta}}|$ is motivated by the fact that the degrees of freedom for an estimate involving the $\ell_1$ penalty can be approximated by the cardinality of the estimated parameter \citep{zou2007degrees}.}  

Using a similar idea, we propose the following BIC-type quantity for selecting the set of tuning parameters $(\lambda_1, \lambda_2, \lambda_3)$ for (\ref{Eq:ggmhub}):
\[
\mathrm{BIC} (\hat{\mathbf{\Theta}},\hat{\mathbf{V}},\hat{\mathbf{Z}}) =-n \cdot \log \det (\hat{\mathbf{\Theta}}) + n \cdot \text{trace}(\mathbf{S}\hat{\mathbf{\Theta}}) + \log (n) \cdot |\hat{\mathbf{Z}}| + \log (n) \cdot \left(\nu+ c  \cdot [|\hat{\mathbf{V}}| - \nu ]\right),
\]
where $\nu$ is the number of estimated hub nodes, that is, $\nu = \sum_{j=1}^p 1_{\{\|\hat{\mathbf{V}}_j\|_0 >0 \}}$, $c$ is a constant between zero and one, and $|\hat{\mathbf{Z}}|$ and $|\hat{\mathbf{V}}|$ are the cardinalities (the number of unique non-zeros) of $\hat{\mathbf{Z}}$ and $\hat{\mathbf{V}}$, respectively.\footnote{ The term $\log(n) \cdot |\hat{\mathbf{Z}}|$ is motivated by the degrees of freedom from the $\ell_1$ penalty, and the term $\log (n) \cdot \left(\nu+ c  \cdot [|\hat{\mathbf{V}}| - \nu ]\right)$ is motivated by an approximation of the degrees of freedom of the $\ell_2$ penalty proposed in \citet{grouplasso}.}
We select the set of tuning parameters $(\lambda_1,\lambda_2,\lambda_3)$ for which the quantity BIC$(\hat{\mathbf{\Theta}},\hat{\mathbf{V}},\hat{\mathbf{Z}})$ is minimized.
Note that when the constant $c$ is small, BIC$(\hat{\mathbf{\Theta}},\hat{\mathbf{V}},\hat{\mathbf{Z}})$ will favor more hub nodes in $\hat{\mathbf{V}}$.  In this manuscript, we take $c=0.2$.

\subsection{Simulation Study}
\label{GGM:simulation}

In this section, we compare HGL to two sets of proposals: proposals that learn an Erd\H{o}s-R\'{e}nyi Gaussian graphical model, and proposals that learn a Gaussian graphical model in which some nodes are highly-connected. 

\subsubsection{Notation and Measures of Performance}
\label{GGM:metric}
    We start by defining some notation.  Let $\hat{\mathbf{\Theta}}$ be the estimate of ${\bf \Theta}={\bf \Sigma}^{-1}$ from a given proposal,  and let $\hat{\mathbf{\Theta}}_j$ be  its $j$th column.  Let $\mathcal{H}$ denote the set of indices of the hub nodes in $\mathbf{\Theta}$ (that is, this is the set of true hub nodes in the graph), and let $|\mathcal{H}|$ denote the cardinality of the set.  In addition, let $\hat{\mathcal{H}}_r$ be the set of \emph{estimated hub nodes}: the set of nodes in $\hat{\bf \Theta}$ that are among the $|\mathcal{H}|$ most highly-connected nodes, and that have at least $r$ edges.  The values chosen for $|\mathcal{H}|$ and $r$ depend on the simulation set-up, and will be specified in each simulation study. 

We now define several measures of performance that will be used  to evaluate the various methods.
\begin{itemize}
\item Number of correctly estimated edges: $\sum_{j <j'}  \left(1_{\{ |\hat{\Theta}_{jj'}| > 10^{-5} \text{ and }  {  | {\Theta}_{jj'}}| \ne 0  \} }\right)$.
\item Proportion of correctly estimated hub edges:
$$\frac{\sum_{j\in \mathcal{H}, j' \ne j}  \left(1_{\{ |\hat{\Theta}_{jj'}| > 10^{-5} \text{ and }  |\Theta_{jj'}| \ne 0  \} }\right)}{\sum_{j \in \mathcal{H}, j' \ne j}  \left( 1_{\{ |\Theta_{jj'}| \ne 0  \} } \right) }.$$

\item Proportion of correctly estimated hub nodes:$ \frac{|\hat{\mathcal{H}}_r \cap \mathcal{H} |}{|\mathcal{H}|}$.

\item Sum of squared errors: $\sum_{j<j'} \left(\hat{\Theta}_{jj'} - \Theta_{jj'}\right)^2$.
\end{itemize}

\subsubsection{Data Generation}
\label{GGM:datagenerate}

We consider three set-ups for generating a $p\times p$ adjacency matrix $\mathbf{A}$.  

\begin{enumerate}[I -]
\item Network with  hub nodes: for all $i<j$, we set $A_{ij}=1$ with probability 0.02, and zero otherwise. We then set $A_{ji}$ equal to  $A_{ij}$.  Next, we randomly select $| \mathcal{H} |$ hub nodes and set the elements of the corresponding rows and columns of $\mathbf{A}$ to equal one with probability 0.7 and zero otherwise.

\item Network with two connected components and hub nodes:  the adjacency matrix is generated as $\mathbf{A}=\begin{pmatrix} \mathbf{A}_1 &0    \\ 0 & \mathbf{A}_2   \end{pmatrix}$, with $\mathbf{A}_1$ and $\mathbf{A}_2$  as in Set-up I, each with $|\mathcal{H}|/2$ hub nodes.

\item Scale-free network:\footnote{Recall that our proposal is not intended for estimating a scale-free network.} the probability that a given node has $k$ edges is proportional to $k^{-\alpha}$.  \cite{barabasi1999} observed that many real-world networks have $\alpha \in [2.1, 4]$; we took $\alpha=2.5$.   Note that there is no natural notion of  hub nodes in a scale-free network.  While some nodes in a scale-free network have more edges than one would expect in an Erd\H{o}s-R\'{e}nyi graph, there is no clear distinction between ``hub" and ``non-hub" nodes, unlike in Set-ups I and II.
In our simulation settings, we consider any node that is connected to more than 5\% of all other nodes to be a  hub node.\footnote{The cutoff threshold of 5\% is chosen  in order to capture the most highly-connected nodes in the scale-free network.  In our simulation study, around three nodes are connected to at least $0.05 \times p$ other nodes in the network. The precise choice of  cut-off threshold has little effect on the results obtained in the figures that follow.}\end{enumerate}

\noindent We then use the adjacency matrix $\bf A$ to create a matrix $\mathbf{E}$, as
\begin{equation*}
E_{ij} \stackrel{\mathrm{i.i.d.}} \sim
\begin{cases}
0 & \text{if } A_{ij}=0\\
\text{Unif}([-0.75, -0.25] \cup [0.25, 0.75])  & \text{otherwise},\\
\end{cases}
\end{equation*}
and  set $\bar{\mathbf{E}}= \frac{1}{2}(\mathbf{E}+\mathbf{E}^T)$.
Given the matrix $\bar{\mathbf{E}}$, we set $\mathbf{\Sigma}^{-1}$ equal to $\bar{\mathbf{E}}+(0.1-\Lambda_{\min}({ \bar{\mathbf{E}}}))\mathbf{I}$, where $\Lambda_{\min}({ \bar{\mathbf{E}}})$ is the smallest eigenvalue of  $\bar{\mathbf{E}}$.  We generate the data matrix $\mathbf{X}$  according to $\mathbf{x}_1,\ldots,\mathbf{x}_n \stackrel{\mathrm{i.i.d.}} \sim N(\mathbf{0}, \mathbf{\Sigma})$.
Then, variables are standardized to have standard deviation one.

\subsubsection{Comparison to  Graphical Lasso and Neighbourhood Selection}
\label{GGM:results}

In this subsection, we compare the performance of HGL to two proposals that learn a sparse Gaussian graphical model. 
\begin{itemize}
\item The graphical lasso  (\ref{Eq:l1penalizeggm}), implemented using the \verb=R= package \verb=glasso=.
\item The neighborhood selection approach of \citet{mb2006}, implemented using the \verb=R= package \verb=glasso=.  This approach involves performing $p$ $\ell_1$-penalized regression problems, each of which involves regressing one feature onto the others. 
\end{itemize}

We consider  the three simulation set-ups described in the previous section with $n=1000$, $p=1500$, and $|\mathcal{H}|=30$ hub nodes in Set-ups I and II.  Figure~\ref{Fig:simulation1} displays the results, averaged over 100 simulated data sets.  Note that the sum of squared errors is not computed for \citet{mb2006}, since it does not directly yield an estimate of $\mathbf{\Theta}=\mathbf{\Sigma}^{-1}$.  

HGL has three tuning parameters. To obtain the curves shown in Figure~\ref{Fig:simulation1}, we fixed $\lambda_1=0.4$, considered three values of $\lambda_3$ (each shown in a different color in Figure~\ref{Fig:simulation1}), and used a fine grid of values of $\lambda_2$. The solid black circle in Figure~\ref{Fig:simulation1} corresponds to the set of tuning parameters $(\lambda_1,\lambda_2,\lambda_3)$  for which the BIC as defined in Section~\ref{Sec:tuning parameter} is minimized.   The graphical lasso and \citet{mb2006} each involves one tuning parameter; we applied them  using a fine grid of the tuning parameter to obtain the curves shown in Figure~\ref{Fig:simulation1}.  

Results for Set-up I are displayed in Figures~\ref{Fig:simulation1}-I(a) through \ref{Fig:simulation1}-I(d), where we calculate the proportion of correctly estimated hub nodes as defined in Section~\ref{GGM:metric} with $r=300$.    Since this simulation set-up exactly matches the assumptions of HGL, it is not surprising that HGL outperforms the other methods.  In particular, HGL is able to identify most of the hub nodes when the number of estimated edges is approximately equal to the true number of edges.  We see similar results for Set-up II in Figures~\ref{Fig:simulation1}-II(a) through \ref{Fig:simulation1}-II(d), where the proportion of correctly estimated hub nodes is as defined in Section~\ref{GGM:metric} with $r=150$.   

In Set-up III, recall that we define a node that is connected to at least 5\% of all nodes to be a hub.  The proportion of correctly estimated hub nodes is then as defined in Section~\ref{GGM:metric} with $r=0.05\times p$.  The results are presented in Figures~\ref{Fig:simulation1}-III(a) through \ref{Fig:simulation1}-III(d).  In this set-up, only approximately three of the nodes (on average) have more than 50 edges, and the hub nodes are not as highly-connected as in Set-up I or Set-up II.  Nonetheless, HGL outperforms the graphical lasso and \citet{mb2006}.  

Finally, we see from Figure 3 that the set of tuning parameters ($\lambda_1,\lambda_2,\lambda_3$) selected using BIC performs reasonably well.  In particular, the graphical lasso solution always has BIC larger than HGL, and hence, is not selected.  

\begin{figure}[htp]
\begin{center}
\includegraphics[scale=0.51]{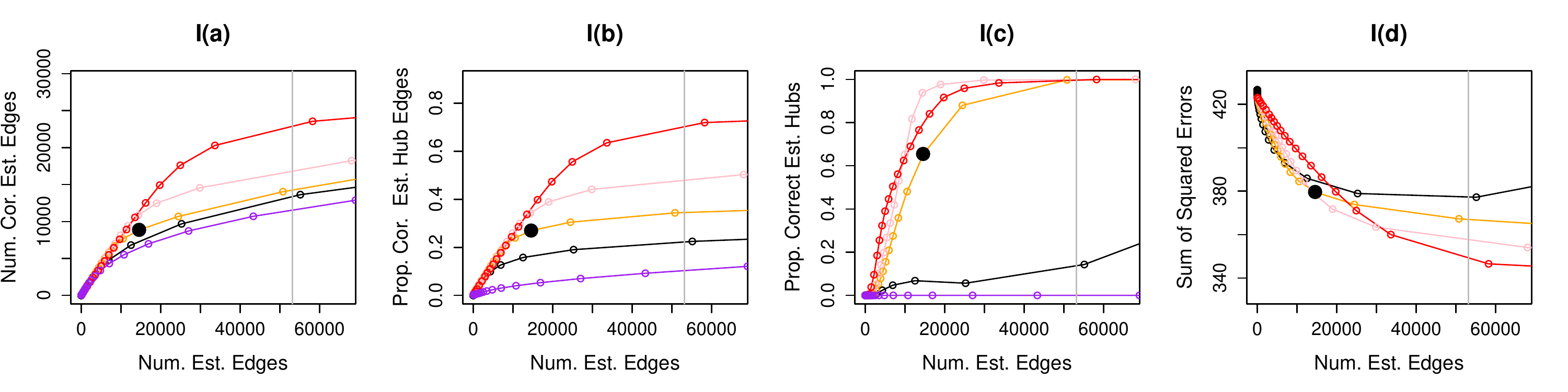}
\includegraphics[scale=0.51]{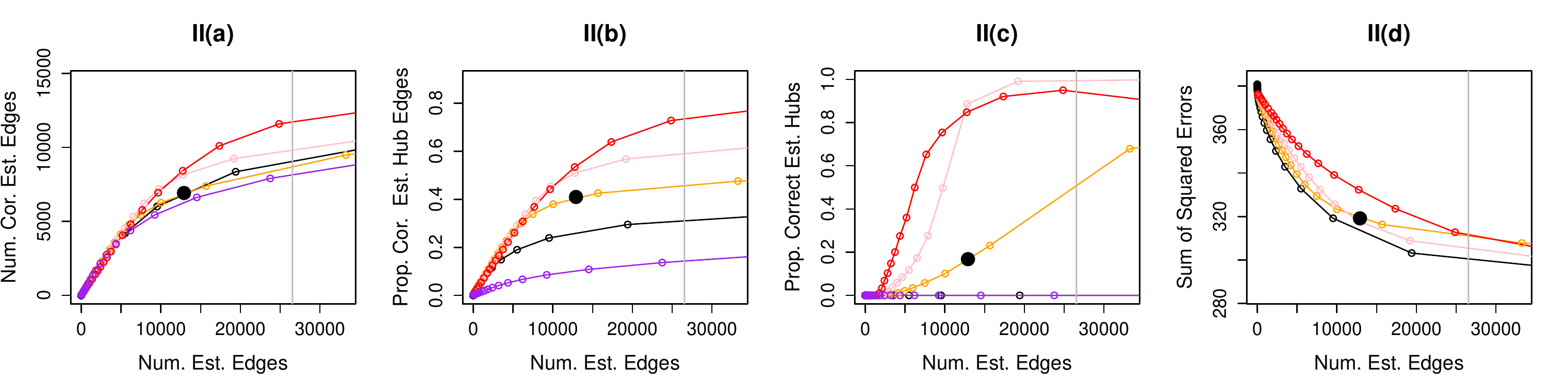}
\includegraphics[scale=0.51]{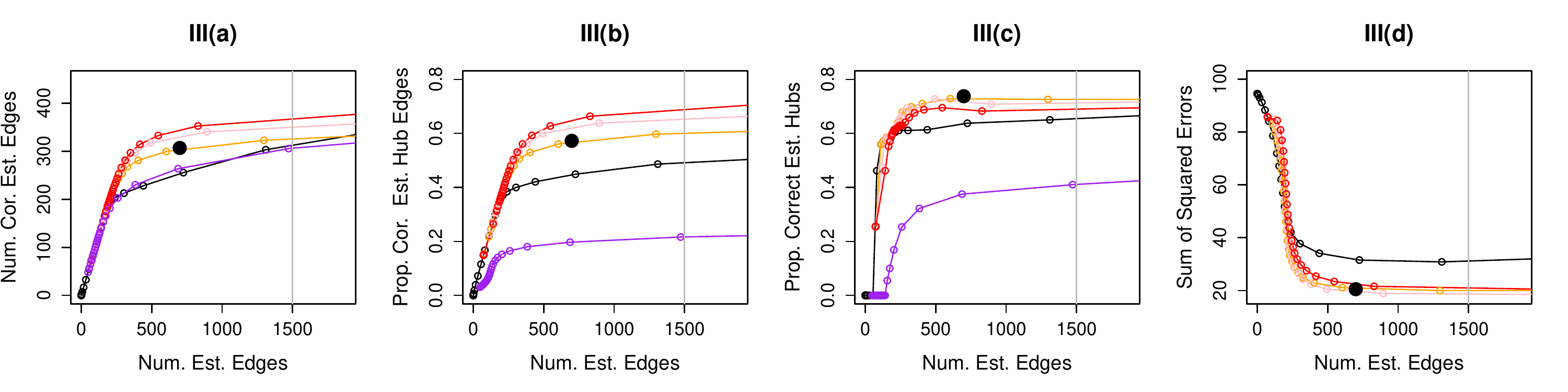}
 \end{center}
  \caption{Simulation for Gaussian graphical model.  Row I:  Results for Set-up I.  Row II: Results for Set-up II.  Row III: Results for Set-up III.  The results are for $n=1000$ and $p=1500$.  In each panel, the $x$-axis displays the number of estimated edges, and the vertical gray line is the number of edges in the true network. The $y$-axes are as follows: Column (a): Number of correctly estimated edges;  Column (b): Proportion of correctly estimated hub edges;  Column (c): Proportion of correctly estimated hub nodes;
  Column (d): Sum of squared errors. The black solid circles are the results for HGL based on tuning parameters selected using the BIC-type criterion defined in Section~\ref{Sec:tuning parameter}. Colored lines correspond to  the graphical lasso {\protect\citep{SparseInv}} (\protect\includegraphics[height=0.4em]{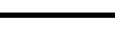}); HGL with $\lambda_3=0.5$ (\protect\includegraphics[height=0.4em]{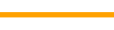}), $\lambda_3=1$ (\protect\includegraphics[height=0.4em]{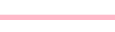}), and $\lambda_3=2$ (\protect\includegraphics[height=0.4em]{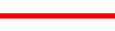}); neighborhood selection  {\protect\citep{mb2006}} (\protect\includegraphics[height=0.4em]{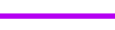}).}
    \label{Fig:simulation1}
\end{figure}%

\subsubsection{Comparison to Additional Proposals}
\label{subSec:other proposal}
In this subsection, we compare the performance of HGL to three additional proposals:
 \begin{itemize}
\item The partial correlation screening procedure of  \cite{HeroandRajaratnam2012}.  The elements of the partial correlation matrix (computed using a pseudo-inverse when $p >n$) are thresholded based on their absolute value, and a    hub node is declared if the number of nonzero elements in the corresponding column of the  thresholded partial correlation matrix is sufficiently large.  Note that the purpose of \citet{HeroandRajaratnam2012} is to screen for hub nodes, rather than to estimate the individual edges in the network.
\item The scale-free network estimation procedure of   \cite{QiangLiu2011}.  This is the solution to the non-convex optimization problem
\begin{equation}
\small 
\label{Eq:LI}
\underset{\mathbf{\Theta}\in \mathcal{S}}{\mathrm{minimize}} \quad \left \{  -\log \det \mathbf{\Theta} + \text{trace}(\mathbf{S\Theta}) + \alpha \sum_{j=1}^p \log(\|\mathbf{\theta}_{\setminus j}\|_1 + \epsilon_j) + \sum_{j=1}^p \beta_j |\theta_{jj}|       \right \}, 
\end{equation}
where $\theta_{\setminus j}= \{\theta_{jj'}| j' \ne j\}$, and $\epsilon_j$, $\beta_j$, and $\alpha$ are tuning parameters.  Here, $\mathcal{S}=\{ \mathbf{\Theta} : \mathbf{\Theta} \succ 0 \text{ and } \mathbf{\Theta}=\mathbf{\Theta}^T  \}$.
 \item Sparse partial correlation estimation procedure of \citet{Space}, implemented using the \verb=R= package \verb=space=.  This is an extension of  the neighborhood selection approach of \citet{mb2006} that combines $p$ $\ell_1$-penalized regression problems in order to obtain a symmetric estimator.  The authors claimed that the proposal performs well in estimating a scale-free network.
 \end{itemize}

We generated data under Set-ups I and III (described in Section~\ref{GGM:datagenerate}) with $n=250$ and $p=500$,\footnote{In this subsection, a small value of $p$ was used due to the computations required to run the R package space, as well as computational demands of the \cite{QiangLiu2011} algorithm.} and with $|\mathcal{H}|=10$ for Set-up I.  The results, averaged over 100 data sets, are displayed in Figures~\ref{Fig:simulation1b} and \ref{Fig:simulation1c}.

To obtain Figures~\ref{Fig:simulation1b} and \ref{Fig:simulation1c}, we applied \citet{QiangLiu2011} using a fine grid of $\alpha$ values, and using the choices for $\beta_j$ and $\epsilon_j$ specified by the authors: $\beta_j = 2 \alpha / \epsilon_j$, where $\epsilon_j$ is a small constant specified in \cite{QiangLiu2011}.
There are two tuning parameters in \citet{HeroandRajaratnam2012}: (1) $\rho$, the value used to threshold the partial correlation matrix, and (2) $d$, the number of non-zero elements required for a column of the  thresholded matrix to be declared a hub node. 
  We used {$d=\{10,20\}$ in Figures~\ref{Fig:simulation1b} and~\ref{Fig:simulation1c}, and used a fine grid of values for $\rho$.  Note that the value of $d$ has no effect on the results for Figures~\ref{Fig:simulation1b}(a)-(b) and Figures~\ref{Fig:simulation1c}(a)-(b), and that larger values of $d$ tend to yield worse results in 
  Figures~\ref{Fig:simulation1b}(c) and \ref{Fig:simulation1c}(c).  For \citet{Space}, we used a fine grid of tuning parameter  values to obtain the curves shown in Figures~\ref{Fig:simulation1b} and~\ref{Fig:simulation1c}. The sum of squared errors was not reported for \citet{Space} and \citet{HeroandRajaratnam2012}  since they do not directly yield an estimate of the precision matrix.  As a baseline reference, the graphical lasso is included in the comparison.       

We see from Figure~\ref{Fig:simulation1b} that HGL outperforms the competitors when the underlying network contains hub nodes.  It is not surprising that  \citet{QiangLiu2011} yields better results than the graphical lasso,  since the former approach is implemented via an iterative procedure: in each iteration, the graphical lasso is performed with an updated tuning parameter based on the estimate obtained in the previous iteration.  \citet{HeroandRajaratnam2012}  has the worst results in Figures~\ref{Fig:simulation1b}(a)-(b); this is not surprising,  since the purpose of \citet{HeroandRajaratnam2012} is to screen for hub nodes, rather than to estimate the individual edges in the network.    

 {From Figure~\ref{Fig:simulation1c}, we see that the performance of HGL is comparable to that of \citet{QiangLiu2011} and \citet{Space} under the assumption of a scale-free network; note that this is the precise setting for which \citet{QiangLiu2011}'s proposal is intended, and \citet{Space} reported that their proposal performs well in this setting.  In contrast, HGL is not intended for the scale-free network setting (as mentioned in the Introduction, it is intended for a setting with hub nodes). Again, \citet{QiangLiu2011} and \citet{Space} outperform the graphical lasso, and \cite{HeroandRajaratnam2012} has the worst results in Figures~\ref{Fig:simulation1c}(a)-(b).  
Finally, we see from Figures~\ref{Fig:simulation1b} and~\ref{Fig:simulation1c} that  the BIC-type criterion for HGL proposed in Section~\ref{Sec:tuning parameter} yields good results.

\begin{figure}[htp]
\begin{center}
\includegraphics[scale=0.51]{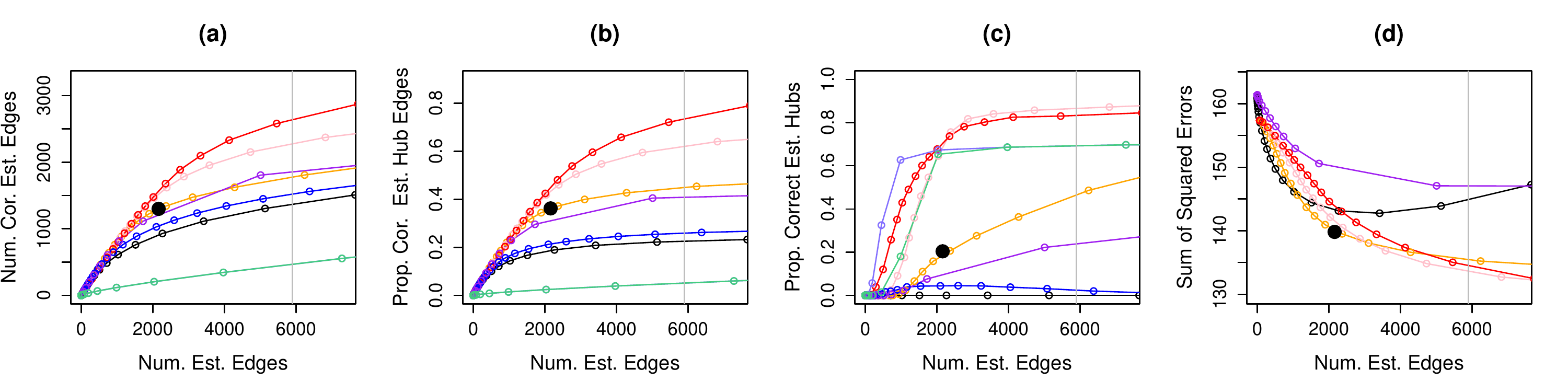}
 \end{center}
  \caption{Simulation for the Gaussian graphical model. Set-up I was applied with $n=250$ and $p=500$. Details of the axis labels and the solid black circles are as in Figure~\ref{Fig:simulation1}. The colored lines correspond to the  graphical lasso {\protect\citep{SparseInv}} (\protect\includegraphics[height=0.5em]{black.png}); HGL with $\lambda_3=1$ (\protect\includegraphics[height=0.5em]{orange.png}), $\lambda_3=2$ (\protect\includegraphics[height=0.5em]{pink.png}), and $\lambda_3=3$ (\protect\includegraphics[height=0.5em]{red.png}); the hub screening procedure  {\protect\citep{HeroandRajaratnam2012}}  with $d=10$ (\protect\includegraphics[height=0.5em]{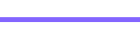}) and $d=20$  (\protect\includegraphics[height=0.5em]{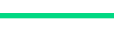}); the scale-free network approach {\protect\citep{QiangLiu2011}} (\protect\includegraphics[height=0.5em]{purple.png}); sparse partial correlation estimation {\protect\citep{Space}} (\protect\includegraphics[height=0.5em]{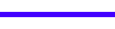}).}
  \label{Fig:simulation1b}
\end{figure}%

\begin{figure}[htp]
\begin{center}
\includegraphics[scale=0.51]{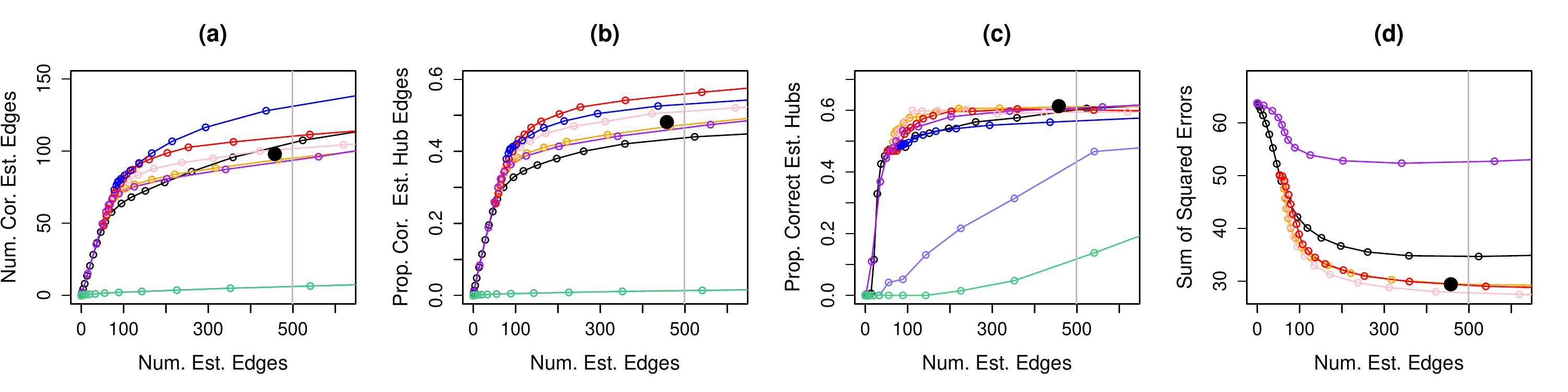}
 \end{center}
  \caption{Simulation for the Gaussian graphical model. Set-up III was applied with $n=250$ and $p=500$.  Details of the axis labels and the solid black circles are as in Figure~\ref{Fig:simulation1}. The colored lines correspond to the  graphical lasso {\protect\citep{SparseInv}} (\protect\includegraphics[height=0.5em]{black.png}); HGL with $\lambda_3=1$ (\protect\includegraphics[height=0.5em]{orange.png}), $\lambda_3=2$ (\protect\includegraphics[height=0.5em]{pink.png}), and $\lambda_3=3$ (\protect\includegraphics[height=0.5em]{red.png}); the hub screening procedure  {\protect\citep{HeroandRajaratnam2012}}  with $d=10$ (\protect\includegraphics[height=0.5em]{slateblue1.png}) and $d=20$  (\protect\includegraphics[height=0.5em]{green.png}); the scale-free network approach {\protect\citep{QiangLiu2011}} (\protect\includegraphics[height=0.5em]{purple.png}); sparse partial correlation estimation {\protect\citep{Space}} (\protect\includegraphics[height=0.5em]{blue.png}).} 
   \label{Fig:simulation1c}
\end{figure}%

\section{The Hub Covariance Graph}
\label{Sec:Covariance}

In this section, we consider estimation of  a covariance matrix under the assumption that $\mathbf{x}_1,\ldots,\mathbf{x}_n \stackrel{\small \mathrm{i.i.d.}}\sim  N(\mathbf{0},\mathbf{\Sigma})$; this is of interest because the sparsity pattern of  $\mathbf{\Sigma}$ specifies the structure of the marginal independence graph  \citep[see, e.g.,][]{Drton2003covgraph,Chaudhurietal2007,DrtonRichardson08}.  We extend the  covariance estimator of \citet{Xueetal2012} to accommodate hub nodes.

\subsection{Formulation and Algorithm}
\label{Cov:formulation}

 \citet{Xueetal2012} proposed to estimate $\bf \Sigma$ using 
\begin{equation}
\label{Eq:CovXue}
\hat{\mathbf{\Sigma}} = \underset{\mathbf{\Sigma} \in \mathcal{S}}{\arg\min}\left\{  \frac{1}{2} \| \mathbf{\Sigma}- \mathbf{S}  \|_F^2 + \lambda \|\mathbf{\Sigma}\|_1 \right\},
\end{equation}
where $\mathbf{S}$ is the empirical covariance matrix, $\mathcal{S} = \{\mathbf{\Sigma}: \mathbf{\Sigma} \succeq \epsilon \mathbf{I} \text{ and } \mathbf{\Sigma}=\mathbf{\Sigma}^T  \}$, and $\epsilon$ is a small positive constant; we take $\epsilon  = 10^{-4}$.   
We extend  (\ref{Eq:CovXue}) to accommodate hubs  by imposing the hub penalty function (\ref{Eq:hubpenalty})  on  $\mathbf{\Sigma}$. This results in the \emph{hub covariance graph} (HCG) optimization problem,
\begin{equation*}
 \underset{\mathbf{\Sigma} \in \mathcal{S} } {\text{minimize}} \qquad 
  \left\{ \frac{1}{2} \|\mathbf{\Sigma}-\mathbf{S}\|_F^2  + \text{P}(\mathbf{\Sigma})  \right\},
 \end{equation*}
which can be solved via Algorithm~\ref{Alg:general}.  To update $\mathbf{\Theta}=\mathbf{\Sigma}$ in Step 2(a)i, we note that

\begin{equation*}
\underset{\mathbf{\Sigma} \in \mathcal{S}}{\arg \min} \left \{  \frac{1}{2} \|\mathbf{\Sigma}-\mathbf{S}\|_F^2  + \frac{\rho}{2} \| \mathbf{\Sigma}-\tilde{\mathbf{\Sigma}}+\mathbf{W}_1  \|_F^2 \right \}
= \frac{1}{1+\rho} (\mathbf{S}+\rho\tilde{\mathbf{\Sigma}}-\rho \mathbf{W}_1)^+,
\end{equation*}
where $(\mathbf{A})^+$ is the projection of a matrix $\mathbf{A}$ onto the convex cone $\{ \mathbf{\Sigma} \succeq \epsilon \mathbf{I} \}$.  That is, if  $\sum_{j=1}^p d_j \mathbf{u}_j \mathbf{u}_j^T$ denotes the eigen-decomposition of the matrix $\mathbf{A}$, then $(\mathbf{A})^+$ is defined as \\$\sum_{j=1}^p \max(d_j,\epsilon) \mathbf{u}_j \mathbf{u}_j^T$.  The complexity of the ADMM algorithm is $O(p^3)$ per iteration, due to the complexity of the eigen-decomposition for updating $\mathbf{\Sigma}$.

\subsection{Simulation Study}
\label{Cov:simulation}
We compare HCG to two competitors for obtaining a sparse estimate of $\mathbf{\Sigma}$:
\begin{enumerate}
\item The non-convex $\ell_1$-penalized log-likelihood approach of  \citet{BienTibs11}, using the \verb=R= package \verb=spcov=.  This approach solves
\begin{equation*}
\label{Eq:Bien}
 \underset{\mathbf{\Sigma} \succ 0 } {\text{minimize}} 
  \left \{ \log \det \mathbf{\Sigma} + \text{trace}(\mathbf{\Sigma}^{-1}\mathbf{S}) + \lambda \|\mathbf{\Sigma}\|_1\right\}.
 \end{equation*}

\item The convex $\ell_1$-penalized approach of \citet{Xueetal2012}, given in (\ref{Eq:CovXue}).
\end{enumerate}

We first generated an adjacency matrix $\mathbf{A}$  as in Set-up I in Section~\ref{GGM:datagenerate}, modified to have $|\mathcal{H}| = 20$ hub nodes.  Then $\bar{\mathbf{E}}$ was  generated as described in Section~\ref{GGM:datagenerate}, and we set $\mathbf{\Sigma}$ equal to $\bar{\mathbf{E}} +(0.1-\Lambda_{\min}(\bar{\mathbf{E}}))\mathbf{I}$.   Next, we generated $\mathbf{x}_1, \ldots, \mathbf{x}_n \stackrel{\small \mathrm{i.i.d.}} \sim N(\mathbf{0},\mathbf{\Sigma})$.  Finally, we standardized the variables to have standard deviation one. In this simulation study, we set $n=500$ and $p=1000$.  

 Figure~\ref{Fig:CovSim1} displays the results, averaged over 100 simulated data sets.  We calculated the proportion of correctly estimated hub nodes as defined in Section 3.3.1 with $r=200$.  We used a fine grid of tuning parameters for \citet{Xueetal2012} in order to obtain the curves shown in each panel of Figure~\ref{Fig:CovSim1}.  HCG involves three tuning parameters, $\lambda_1$, $\lambda_2$, and $\lambda_3$.  We fixed $\lambda_1 = 0.2$, considered three values of $\lambda_3$ (each shown in a different color), and varied $\lambda_2$ in order to obtain the curves shown in Figure~\ref{Fig:CovSim1}.  
 
 Figure~\ref{Fig:CovSim1} does not display the results for the proposal of \citet{BienTibs11}, due to computational constraints in the \verb=spcov= \verb=R= package. Instead, we compared our proposal to that of \citet{BienTibs11} using $n=100$ and $p=200$; those  results are presented in Figure~\ref{Fig:CovSimsmall} in  Appendix D.

\begin{figure}[htp]
\begin{center}
\includegraphics[scale=0.51]{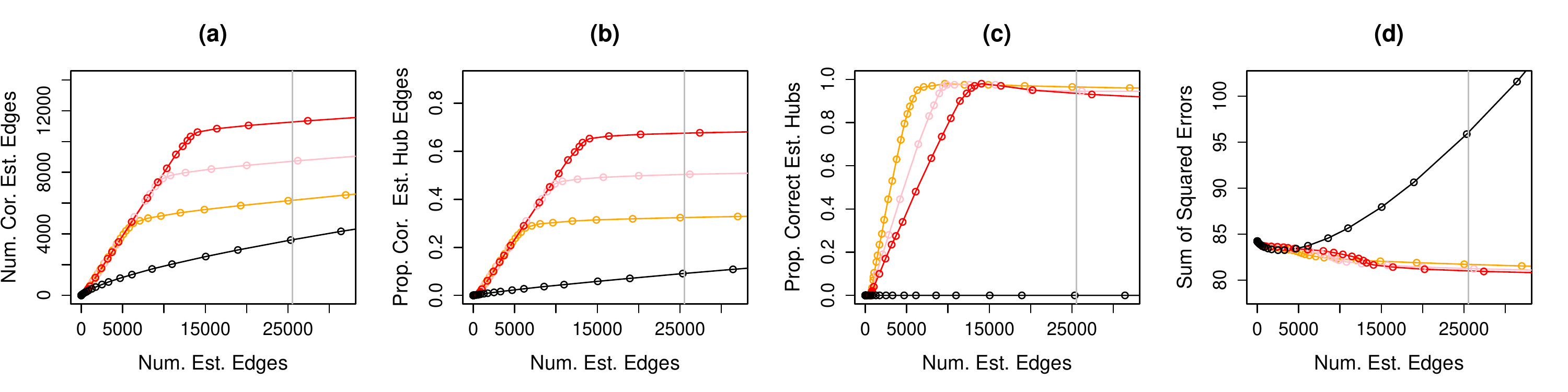}
 \end{center}
  \caption{Covariance graph simulation with $n=500$ and $p=1000$.  Details of the axis labels are as in Figure~\ref{Fig:simulation1}.  The colored lines correspond to the proposal of {\protect\citet{Xueetal2012}} (\protect\includegraphics[height=0.5em]{black.png}); HCG with $\lambda_3=1$ (\protect\includegraphics[height=0.5em]{orange.png}), $\lambda_3=1.5$ (\protect\includegraphics[height=0.5em]{pink.png}), and  $\lambda_3=2$ (\protect\includegraphics[height=0.5em]{red.png}).}  
  \label{Fig:CovSim1}
\end{figure}

We see that HCG outperforms the proposals of \citet{Xueetal2012} (Figures~\ref{Fig:CovSim1} and \ref{Fig:CovSimsmall}) and \citet{BienTibs11} (Figure~\ref{Fig:CovSimsmall}).  These results are not surprising,  since those other methods do not explicitly model the hub nodes.

\section{The Hub Binary Network}
\label{Sec:Binary}

In this section, we focus on estimating a binary Ising Markov random field, which we refer to as a binary network. We refer the reader to \citet{ravikumaretal2010} for an in-depth discussion of this type of graphical model and its applications. 

In this set-up, each entry of the $n\times p$ data matrix $\mathbf{X}$ takes on a  value of zero or one.  We assume that the observations $\mathbf{x}_1,\ldots,\mathbf{x}_n$ are i.i.d. with density
\begin{equation}
\label{Eq:Isingmodel}
p(\mathbf{x},\mathbf{\Theta}) = \frac{1}{Z(\mathbf{\Theta})}\exp  \left[ \sum_{j=1}^p \theta_{jj} x_j  +\sum_{1\le j < j' \le p} \theta_{jj'} x_j x_{j'}   \right],
\end{equation}
 where $Z({\mathbf{\Theta}})$  is the partition function, which ensures that the density sums  to one.  
Here $\mathbf{\Theta}$ is a $p\times p$ symmetric matrix that specifies the network structure: $\theta_{jj'}=0$ implies that the $j$th and $j'$th variables are conditionally independent. 

In order to obtain a sparse graph, \citet{LeeSIetal2007} considered maximizing an $\ell_1$-penalized log-likelihood under this model. Due to the difficulty in computing the log-partition function, several authors have considered alternative approaches.  For instance, \citet{ravikumaretal2010} proposed a neighborhood selection approach.  The proposal of \citet{ravikumaretal2010} involves solving $p$ logistic regression separately, and hence, the estimated parameter matrix is not symmetric.  In contrast, several authors considered maximizing an  $\ell_1$-penalized  pseudo-likelihood with a symmetric constraint on $\mathbf{\Theta}$ \citep[see, e.g.,][]{Hoefling2009,jianguobinary,guoasymptotic}.

\subsection{Formulation and Algorithm}
\label{Binary:formulation}
Under the model   (\ref{Eq:Isingmodel}), the log-pseudo-likelihood for $n$ observations takes the form 
\begin{equation}
\label{Eq:logpseudo}
\sum_{j=1}^p \sum_{j'=1}^p \theta_{jj'} (\mathbf{X}^T\mathbf{X})_{jj'} - \sum_{i=1}^n\sum_{j=1}^p \log \left( 1+   \text{exp}\left[\theta_{jj}+ \sum_{j'\ne j} \theta_{jj'}x_{ij'}\right] \right),
\end{equation}
 where $\mathbf{x}_i$ is the $i$th row of the $n\times p$ matrix $\mathbf{X}$.  The proposal of \citet{Hoefling2009} involves maximizing  (\ref{Eq:logpseudo}) subject to
  an $\ell_1$ penalty on $\bf \Theta$. We propose to instead impose the  hub penalty function (\ref{Eq:hubpenalty}) on $\mathbf{\Theta}$ in (\ref{Eq:logpseudo}) in order to estimate a sparse binary network with hub nodes.  
This leads to the optimization problem
\begin{eqnarray}
\small
\label{Eq:binaryformulation}
\begin{aligned}
 &\underset{\mathbf{\Theta} \in \mathcal{S}} {\text{minimize}} 
 && \left \{  -\sum_{j=1}^p \sum_{j'=1}^p \theta_{jj'} (\mathbf{X}^T\mathbf{X})_{jj'} + \sum_{i=1}^n\sum_{j=1}^p \log \left( 1+   \text{exp}\left[\theta_{jj}+ \sum_{j'\ne j} \theta_{jj'}x_{ij'}\right] \right)  + \text{P}(\mathbf{\Theta})  \right\},
\end{aligned}
 \end{eqnarray}
 where $\mathcal{S} = \{\mathbf{\Theta} : \mathbf{\Theta}=\mathbf{\Theta}^T \}$. 
We refer to the solution to (\ref{Eq:binaryformulation}) as the \emph{hub binary  network} (HBN). 
The ADMM algorithm for solving (\ref{Eq:binaryformulation}) is given in Algorithm~\ref{Alg:general}. We solve the update 
for $\mathbf{\Theta}$ in Step 2(a)i 
using the Barzilai-Borwein method \citep{barzilai1988two}.  The details are given in Appendix F.

\subsection{Simulation Study}
\label{Binary:simulation}

Here we compare the performance of HBN to the proposal of \citet{Hoefling2009}, implemented using the \verb=R= package \verb=BMN=.

 We simulated a binary network with $p=50$ and $|\mathcal{H}| = 5$ hub nodes. To generate the  parameter matrix $\mathbf{\Theta}$, we created  an adjacency matrix $\mathbf{A}$  as in Set-up I of Section~\ref{GGM:datagenerate} with five hub nodes.  Then $\bar{\mathbf{E}}$ was generated as in Section~\ref{GGM:datagenerate}, and we set $\mathbf{\Theta}= \bar{\mathbf{E}}$.  

  Each of $n=100$ observations was generated using Gibbs sampling \citep{ravikumaretal2010,jianguobinary}.  Suppose that $x_1^{(t)},\ldots, x_p^{(t)}$ is obtained at the $t$th iteration of the Gibbs sampler.  Then, the $(t+1)$th iteration is obtained according to 
  \[
x_{j}^{(t+1)} \sim \text{Bernoulli} \left(  \frac{\exp(\theta_{jj} + \sum_{j\ne j'} \theta_{jj'} x_{j'}^{(t)})}{1+\exp(\theta_{jj} + \sum_{j\ne j'} \theta_{jj'} x_{j'}^{(t)})}    \right) \qquad \text{for } j=1,\ldots, p.
\] 
We took the first $10^5$ iterations as our burn-in period, and then collected an observation every $10^4$ iterations, such that the observations were nearly independent \citep{jianguobinary}.

The results, averaged over 100 data sets, are shown in Figure~\ref{Fig:BinarySim1}.  We used a fine grid of values for the $\ell_1$ tuning parameter for \citet{Hoefling2009}, resulting in curves shown in each panel of the figure.  For HBN, we fixed $\lambda_1=5$, considered $\lambda_3 = \{15,25,30\}$, and used a fine grid of values of $\lambda_2$.  The proportion of correctly estimated hub nodes was calculated using the definition in Section~\ref{GGM:metric} with $r=20$.         Figure~\ref{Fig:BinarySim1} indicates that HBN  consistently outperforms the proposal of  \citet{Hoefling2009}.

\begin{figure}[htp]
\begin{center}
\includegraphics[scale=0.51]{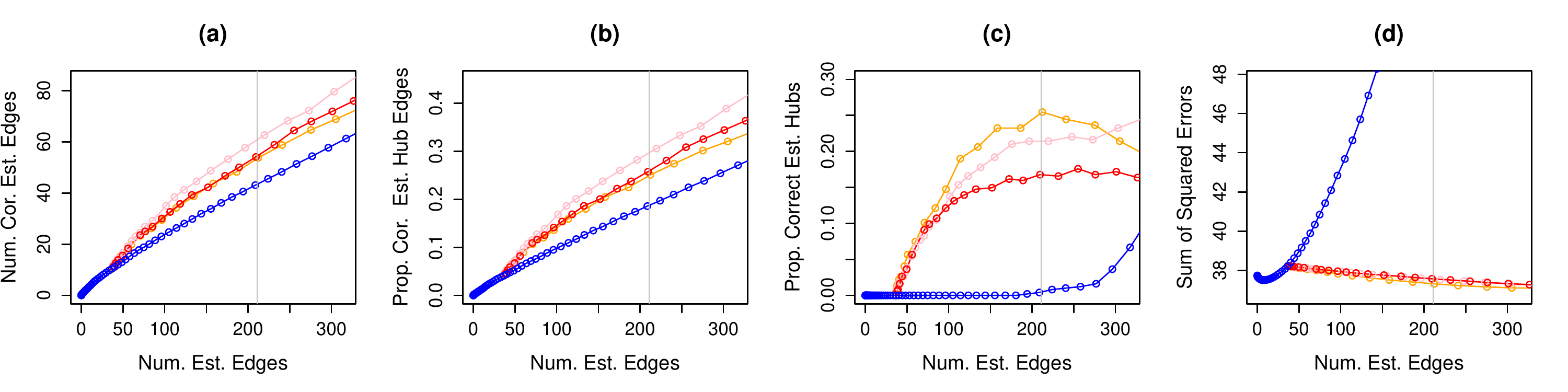}
 \end{center}
  \caption{Binary network simulation with $n=100$ and $p=50$.  Details of the axis labels are as in  Figure~\ref{Fig:simulation1}.  The colored
    lines correspond to the $\ell_1$-penalized  pseudo-likelihood proposal  of
      \protect\citet{Hoefling2009} 
     (\protect\includegraphics[height=0.5em]{blue.png});  and HBN with 
   $\lambda_3=15$  (\protect\includegraphics[height=0.5em]{orange.png}), $\lambda_3=25$ (\protect\includegraphics[height=0.5em]{red.png}), and $\lambda_3=30$ 
  (\protect\includegraphics[height=0.5em]{pink.png}).}
  \label{Fig:BinarySim1}
\end{figure}

\section{Real Data Application}
\label{Sec:realdata}
We now  apply HGL to  a university webpage data set,  and a brain cancer data set. 

\subsection{Application to University Webpage Data}
\label{sec:real_data_analysis} We applied HGL to the university
webpage data set from the ``World Wide Knowledge Base" project at
Carnegie Mellon University.  This data set was pre-processed by
\citet{webpage2011}. The
data set consists of the occurrences of various terms (words) on  webpages from four computer science departments at
Cornell, Texas, Washington and Wisconsin.  We consider only the 544
student webpages, and select 100 terms with the largest entropy for
our analysis.  In what follows, we model these 100 terms as the nodes in a Gaussian graphical model.

The goal of the analysis is to  understand the relationships among
the terms that appear on the student webpages.  In particular, we wish to identify terms that are hubs.  We are not interested in identifying edges between non-hub nodes.  For this reason, we fix the tuning parameter that controls the sparsity of $\mathbf{Z}$ at $\lambda_1 = 0.45$ such that the matrix $\mathbf{Z}$ is sparse.  In the interest of a graph that is interpretable, we fix $\lambda_3=1.5$ to obtain only a few hub nodes, and then select a value of $\lambda_2$ ranging from 0.1 to 0.5  using the BIC-type criterion presented in Section~\ref{Sec:tuning parameter}.  We performed HGL with the selected tuning parameters $\lambda_1=0.45$, $\lambda_2=0.25$, and $\lambda_3=1.5$.\footnote{The results are qualitatively similar for different values of $\lambda_1$.}   The estimated matrices are shown in Figure
\ref{fig:webpage_V_Z}.

Figure \ref{fig:webpage_V_Z}(a) indicates that six hub nodes are detected:
  \emph{comput}, \emph{research}, \emph{scienc}, \emph{software}, \emph{system}, and \emph{work}.  For instance, the fact that \emph{comput} is a hub indicates that many terms'  occurrences are explained by the occurrence of the word \emph{comput}.
  From Figure \ref{fig:webpage_V_Z}(b), we see that several pairs of terms take on non-zero values in the matrix ${(\bf Z} -  \mathrm{diag}({\bf Z}))$.  These include \emph{(depart, univers)}; \emph{(home, page)}; \emph{(institut, technolog)}; \emph{(graduat, student)}; \emph{(univers, scienc)}, and \emph{(languag,program)}. These results provide an intuitive explanation of the relationships among the terms in the webpages.



\begin{figure}[htp]
\begin{center}
(a) \hspace{85mm} (b)
\includegraphics[scale=0.68]{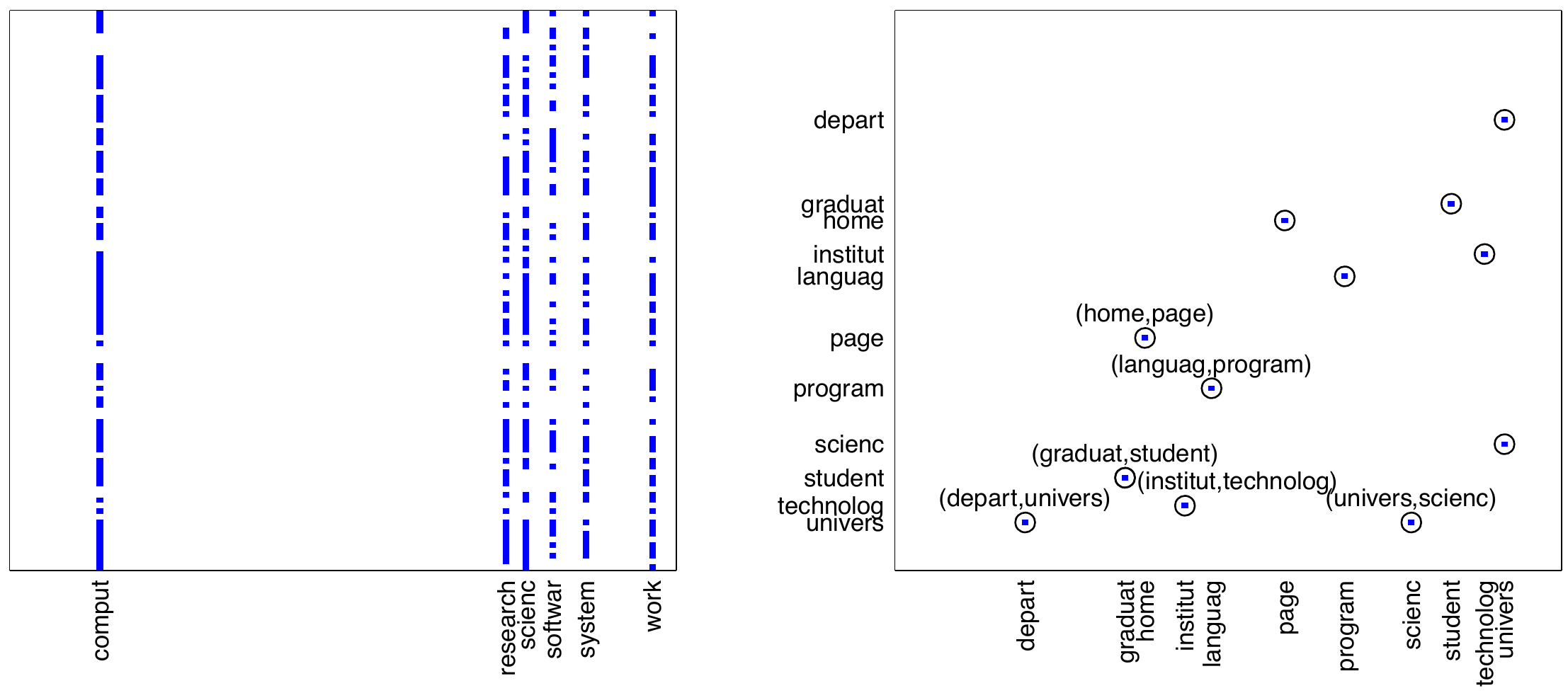}
\end{center}
\caption{Results for HGL on the webpage data with tuning parameters selected using BIC: $\lambda_1 = 0.45$, $\lambda_2 = 0.25$, $\lambda_3 = 1.5$. Non-zero estimated values  are shown, for \emph{(a):}  ${(\bf V} -  \mathrm{diag}({\bf V}))$, and \emph{(b):} ${(\bf Z} -  \mathrm{diag}({\bf Z}))$.} 
\label{fig:webpage_V_Z}
\end{figure}


\subsection{Application to Gene Expression Data}

We applied HGL to a publicly available cancer gene expression
data set \citep{cancer2012}.  The data  set consists of  mRNA expression levels for
17,814 genes in 401 patients with glioblastoma multiforme (GBM), an
extremely aggressive cancer with very poor patient prognosis.  Among 7,462 genes  known to be associated
with cancer \citep{malacards2013}, we selected 500 genes 
 with the highest variance. 
 
We
aim to reconstruct the gene regulatory network that represents the interactions among the genes, as well as to identify hub genes that
tend to have many interactions with other genes. Such genes likely play an important role in regulating many other
genes in the network. Identifying such regulatory genes will  lead to a better
understanding of brain cancer, and
eventually may lead to new therapeutic targets.  Since we are interested in identifying hub genes, and not as interested in identifying edges between non-hub nodes, we fix $\lambda_1=0.6$ such that the matrix $\mathbf{Z}$ is sparse.  We fix $\lambda_3=6.5$ to obtain a few hub nodes, and we select $\lambda_2$ ranging from 0.1 to 0.7 using the BIC-type criterion presented in Section~\ref{Sec:tuning parameter}.  

We applied HGL with this set of tuning parameters to the empirical covariance matrix corresponding to the $401 \times 500$ data matrix, after standardizing each gene to have variance one.
 In Figure~\ref{Figure:gene-network}, we plotted the resulting
network (for simplicity, only the 438 genes with at least two neighbors are displayed). We found that five
 genes are identified as hubs. These genes are TRIM48, TBC1D2B, PTPN2,
ACRC, and ZNF763, in decreasing order of estimated edges.

Interestingly, some of these genes have known regulatory roles.
PTPN2 is known to be a signaling molecule that regulates a variety
of cellular processes including cell growth, differentiation,
mitotic cycle, and oncogenic transformation~\citep{entrez}.
ZNF763 is a DNA-binding protein that regulates the transcription of
other genes~\citep{entrez}. These genes do not appear to be highly-connected to many other genes in the estimate that results from applying the graphical lasso (\ref{Eq:l1penalizeggm}) to this same data set (results not shown).
 These results indicate that HGL can be used to recover known regulators, as well as to suggest other potential regulators that may be targets for follow-up analysis.

\begin{figure}[htp]
\begin{center}
\includegraphics[scale=0.85]{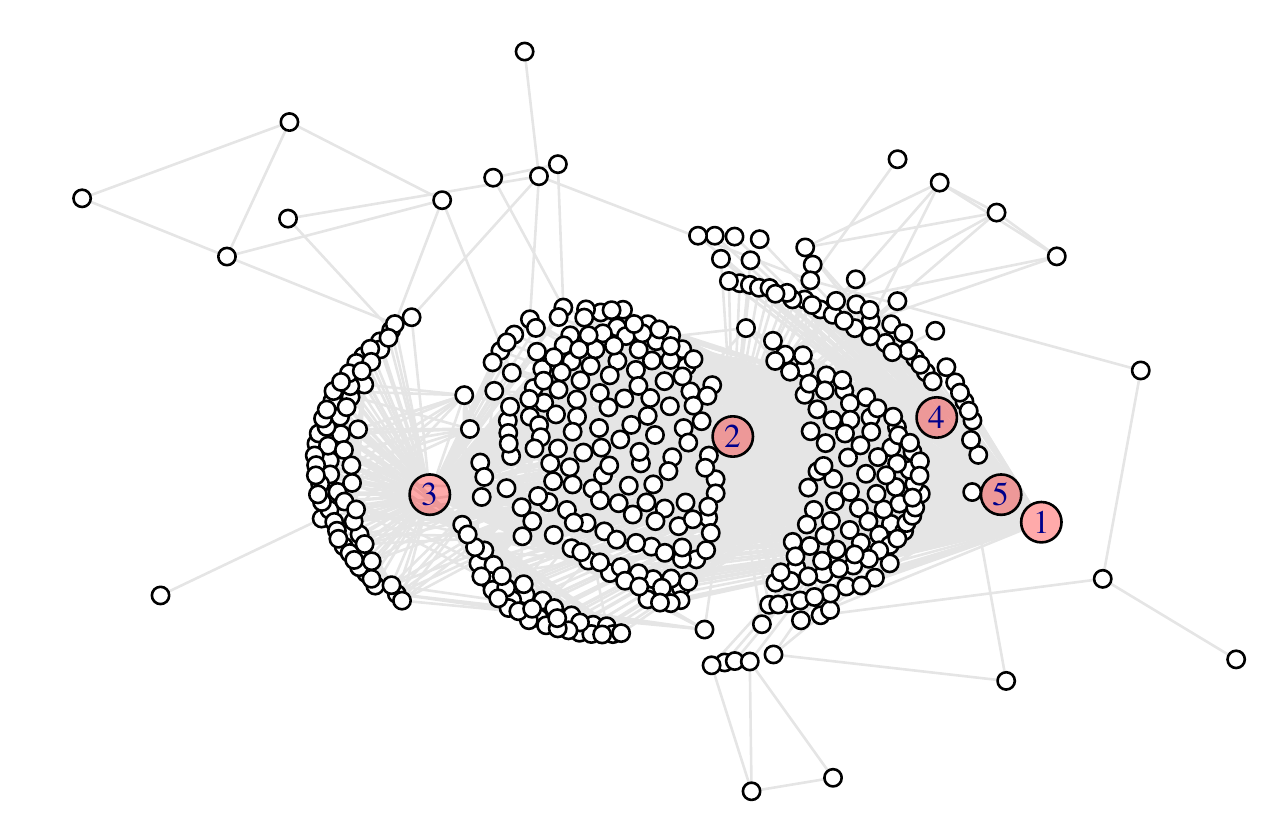}
\includegraphics[scale=0.83]{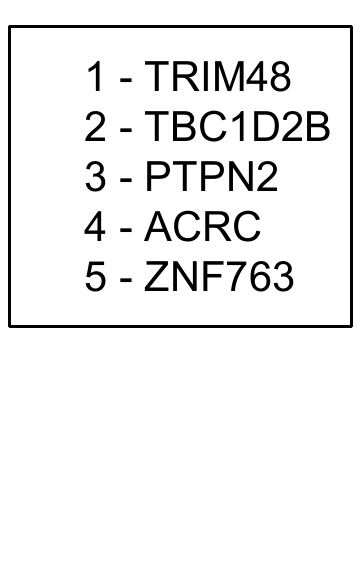}
\end{center}
\caption{ Results for HGL on the GBM data with tuning parameters selected using BIC: $\lambda_1 = 0.6$, $\lambda_2 = 0.4$,
$\lambda_3 = 6.5$. Only nodes with at least two edges in the estimated network are displayed.  Nodes displayed in pink were found to be hubs by the HGL algorithm.}
\label{Figure:gene-network}
\end{figure}


\section{Discussion}
\label{Sec:Discussion}
We have proposed a general framework for estimating a network with hubs by way of a convex penalty function.     The proposed framework has three tuning parameters, so that it can flexibly accommodate different numbers of hubs, sparsity levels within a hub, and connectivity levels among non-hubs.  We have proposed a BIC-type quantity to select tuning parameters for our proposal.  We note that  tuning parameter selection in unsupervised settings remains a challenging open problem \citep[see, e.g.,][]{FoygelDrton10,StabilitySelection}.
  In practice, tuning parameters could also be set based on domain knowledge or a desire for interpretability of the resulting estimates.

The framework proposed in this paper assumes an underlying model involving a set of edges between non-hub nodes, as well as a set of hub nodes. 
For instance, it is believed that such hub nodes arise in biology, in which ``super hubs" in transcriptional regulatory networks may play important roles \citep{Haoetal2012}.
We note here that the underlying model of hub nodes assumed in this paper differs fundamentally from a scale-free network  in which the degree of connectivity of the nodes follows a power law distribution---scale-free networks simply do not have such very highly-connected hub nodes.   In fact, we have shown that existing techniques for estimating a scale-free network, such as
\citet{QiangLiu2011} and \citet{Defazio2012},  cannot accommodate the very dense hubs for which our proposal is intended.

As discussed in Section~\ref{Sec:Penalty}, the hub penalty function involves decomposing a parameter matrix $\bf \Theta$ into ${\bf Z}+{\bf V}+{\bf V}^T$, where $\bf Z$ is a sparse matrix, and $\bf V$ is a matrix whose columns are entirely zero or (almost) entirely non-zero. 
 In this paper, we used an $\ell_1$ penalty on $\bf Z$ in order to encourage it to be sparse. In effect, this amounts to assuming that the non-hub nodes obey an Erd\H{o}s-R\'{e}nyi network. But our formulation could be easily modified to accommodate a different network
 prior for the non-hub nodes. For instance, we could assume that the non-hub nodes  obey a scale-free network, using the ideas  developed in   \citet{QiangLiu2011} and \citet{Defazio2012}. This would amount to modeling a scale-free network with hub nodes.

In this paper, we applied the proposed framework to the tasks of estimating a Gaussian graphical model, a covariance graph model, and a binary network.  The proposed framework can also be applied to other types of graphical models, such as the Poisson graphical model \citep{AllenLiu2012} or the exponential family graphical model \citep{Yangetal2012}.  

In future work, we will study the theoretical statistical properties of the HGL formulation.  For instance, in the context of the graphical lasso, it is known that the rate of statistical convergence depends upon the maximal degree of any node in the network \citep{Ravikumar2011}. 
  It would be interesting to see whether HGL theoretically outperforms the graphical lasso in the setting in which the true underlying network contains  hubs.  Furthermore, it will be of interest to study HGL's hub recovery properties from a theoretical perspective.

An \verb=R= package \verb=hglasso= is publicly available on the authors' websites and on \verb=CRAN=.


\acks{We thank three reviewers for helpful comments that improved the quality of this manuscript.  We thank Qiang Liu helpful responses to our inquiries regarding  \citet{QiangLiu2011}. The authors acknowledge funding from the following sources: NIH DP5OD009145 and NSF CAREER DMS-1252624 and Sloan Research Fellowship  to DW,  NSF CAREER  ECCS-0847077 to MF, and Univ. Washington Royalty Research Fund  to DW, MF, and SL.}


\newpage

\section*{Appendix A: Derivation of Algorithm~\ref{Alg:general}}

Recall that the scaled augmented Lagrangian for (\ref{Eq:reformulate}) takes the form
 \begin{equation}
 \label{Appendix:lagrangian}
\begin{split}
L(\mathbf{B} ,\tilde{\mathbf{B}},\mathbf{W}) &= \ell(\mathbf{X},\mathbf{\Theta})    + \lambda_1 \| \mathbf{Z} - \text{diag}(\mathbf{Z})\|_1+ \lambda_2  \| \mathbf{V} - \text{diag}(\mathbf{V})\|_1 \\
&+\lambda_3 \sum_{j=1}^p \| (\mathbf{V}  - \text{diag}(\mathbf{V}))_j \|_2 +g(\tilde{\mathbf{B}})  +\frac{\rho}{2}\|\mathbf{B}-\tilde{\mathbf{B}}+\mathbf{W}  \|^2_F.\\
\end{split}
\end{equation}

\noindent The proposed ADMM algorithm requires the following updates:
\begin{enumerate}
\item $\mathbf{B}^{(t+1)} \leftarrow \underset{\mathbf{B}}{\text{argmin  }} L(\mathbf{B},\tilde{\mathbf{B}}^{t},\mathbf{W}^{t})$,
\item $\tilde{\mathbf{B}}^{(t+1)} \leftarrow \underset{\tilde{\mathbf{B}}}{\text{argmin  }} L(\mathbf{B}^{(t+1)},\tilde{\mathbf{B}},\mathbf{W}^{t})$,
\item $\mathbf{W}^{(t+1)} \leftarrow \mathbf{W}^{t}+\mathbf{B}^{(t+1)}-\tilde{\mathbf{B}}^{(t+1)}$.
\end{enumerate}

\noindent We now proceed to derive the updates for $\mathbf{B}$ and $\tilde{\mathbf{B}}$.

\subsection*{Updates for $\mathbf{B}$}

\noindent   To obtain updates for $\mathbf{B}=(\mathbf{\Theta,V,Z})$, we exploit the fact that (\ref{Appendix:lagrangian}) is separable in  $\mathbf{\Theta}, \mathbf{V}$, and  $\mathbf{Z}$. Therefore, we can simply update with respect to   $\mathbf{\Theta}, \mathbf{V}$, and  $\mathbf{Z}$ one-at-a-time. Update for $\mathbf{\Theta}$ depends on the form of the convex loss function, and is  addressed in the main text. Updates for $\bf V$ and $\bf Z$ can be easily seen to take the  form given in Algorithm 1.

\subsection*{Updates for $\tilde{\mathbf{B}}$}
Minimizing the function in (\ref{Appendix:lagrangian}) with respect to $\tilde{\mathbf{B}}$ is equivalent to 

 \begin{equation}
 \label{Equation:lagrangian3}
\begin{aligned}
& \underset{{\tilde{\mathbf{\Theta}}},\tilde{\mathbf{V}},\tilde{\mathbf{Z}}        } {\text{minimize}} 
&&\left\{ \frac{\rho}{2}\|\mathbf{\Theta} -\tilde{\mathbf{\Theta}}+ \mathbf{W}_1  \|^2_F +\frac{\rho}{2}\|\mathbf{V} -\tilde{\mathbf{V}}+ \mathbf{W}_2  \|^2_F + \frac{\rho}{2}\|\mathbf{Z} -\tilde{\mathbf{Z}}+ \mathbf{W}_3  \|^2_F\right\}\\
& \text{subject to}
& & \tilde{\mathbf{\Theta}} = \tilde{\mathbf{Z}}+\tilde{\mathbf{V}} + \tilde{\mathbf{V}}^T.
\end{aligned}
\end{equation}
Let $\mathbf{\Gamma}$ be the $p\times p$ Lagrange multiplier matrix for the equality constraint.  Then,  the Lagrangian for (\ref{Equation:lagrangian3}) is
\begin{equation*}
\frac{\rho}{2}\|\mathbf{\Theta} -\tilde{\mathbf{\Theta}}+ \mathbf{W}_1  \|^2_F +\frac{\rho}{2}\|\mathbf{V} -\tilde{\mathbf{V}}+ \mathbf{W}_2  \|^2_F + \frac{\rho}{2}\|\mathbf{Z} -\tilde{\mathbf{Z}}+ \mathbf{W}_3  \|^2_F + \langle \mathbf{\Gamma}  ,\tilde{\mathbf{\Theta}} - \tilde{\mathbf{Z}}-\tilde{\mathbf{V}} - \tilde{\mathbf{V}}^T\rangle.
\end{equation*}
A little bit of algebra yields 
\[
\tilde{\mathbf{\Theta}} = \mathbf{\Theta} + \mathbf{W}_1 - \frac{1}{\rho} \mathbf{\Gamma},
\]
\[
\tilde{\mathbf{V}} = \frac{1}{\rho}( \mathbf{\Gamma+\Gamma}^T) +\mathbf{V} + \mathbf{W}_2,
\]
\[
\tilde{\mathbf{Z}} = \frac{1}{\rho} \mathbf{\Gamma}+ \mathbf{Z} + \mathbf{W}_3,
\]
where  $\mathbf{\Gamma} = \frac{\rho}{6}[(\mathbf{\Theta} + \mathbf{W}_1) -(\mathbf{V+W}_2) - (\mathbf{V+W}_2)^T-(\mathbf{Z+W}_3)]$.

\section*{Appendix B: Conditions for  HGL Solution to be Block-Diagonal }

We begin by introducing some notation.  
Let $\|\mathbf{V}\|_{u,v}$ be the $\ell_u / \ell_v$ norm of a matrix $\mathbf{V}$.  For instance,   
$\| \mathbf{V} \|_{1,q} = \sum_{j=1}^p \| \mathbf{V}_j  \|_q $.  
 We define the support of a matrix $\bf \Theta$ as follows: $\text{supp}(\mathbf{\Theta}) = \{(i,j): \Theta_{ij} \ne 0\}$. We say that $\mathbf{\Theta}$ is supported on a set $\mathcal{G}$ if $\text{supp}(\mathbf{\Theta})\subseteq \mathcal{G}$.
Let $\{C_1,\ldots, C_K\}$ be a partition of the index set $\{1,\ldots,p \}$, and let $\mathcal{T} = \cup_{k=1}^K \{ C_k \times C_k\}$. 
   We let $\mathbf{A}_{\mathcal{T}}$ denote the restriction of the matrix $\mathbf{A}$ to the set $\mathcal{T}$: that is, $(\mathbf{A}_{\mathcal{T}})_{ij}=0$ if $(i,j)\notin \mathcal{T}$ and $(\mathbf{A}_{\mathcal{T}})_{ij}=A_{ij}$ if $(i,j)\in \mathcal{T}$.  Note that any matrix supported on $\mathcal{T}$ is block-diagonal with $K$ blocks,  subject to some permutation of its rows and columns.     Also, let $S_{\max} = \underset{(i,j)\in \mathcal{T}^c}\max |S_{ij}|$.  Define 

\begin{eqnarray}
\label{normequation}
\begin{array}{rccl}
 \tilde{\mathbf{P}}(\mathbf{\Theta}) &=& \underset{{\mathbf{V, Z}}} {\text{min}} &  \| \mathbf{Z}-\text{diag}(\mathbf{Z})\|_1 + \hat{\lambda}_2 \| \mathbf{V}-\text{diag}(\mathbf{V})\|_1 + \hat{\lambda}_3 \|\mathbf{V}-\text{diag}(\mathbf{V}) \|_{1,q}\\
 && \text{subject to} & \mathbf{\Theta} = \mathbf{Z+V+V}^T,
\end{array}
\end{eqnarray}
where $\hat{\lambda}_2=\frac{\lambda_2}{\lambda_1}$ and $\hat{\lambda}_3 = \frac{\lambda_3}{\lambda_1}$.  Then, optimization problem (\ref{Eq:ggmhub}) is equivalent to 
\begin{eqnarray}
\label{Equation:reduceHGL}
 \underset{{\mathbf{\Theta}}\in \mathcal{S}} {\text{minimize}} & -\log \det(\mathbf{\Theta}) +\langle \mathbf{\Theta,S}\rangle + \lambda_1 \tilde{\mathbf{P}}(\mathbf{\Theta}),
\end{eqnarray}
where $\mathcal{S}= \{\mathbf{\Theta}:\mathbf{\Theta} \succ 0 ,\mathbf{\Theta}=\mathbf{\Theta}^T\}$.

\subsection*{Proof of Theorem 1 (Sufficient Condition)}

\begin{proof}
First, we note that if $(\mathbf{\Theta,V,Z})$ is a feasible solution to (\ref{Eq:ggmhub}), then $(\mathbf{\Theta_{\mathcal{T}}},\mathbf{V_{\mathcal{T}}},\mathbf{Z_{\mathcal{T}}}     )$ is also a feasible solution to (\ref{Eq:ggmhub}).  Assume that $(\mathbf{\Theta,V,Z})$ is not supported on $\mathcal{T}$. We want to show that the objective value of (\ref{Eq:ggmhub}) evaluated at  $(\mathbf{\Theta}_{\mathcal{T}},\mathbf{V}_{\mathcal{T}},\mathbf{Z}_{\mathcal{T}})$ is smaller than the objective value of (\ref{Eq:ggmhub}) evaluated at  $(\mathbf{\Theta,V,Z})$.  By Fischer's inequality \citep{HJ85}, 
\[
-\log \det (\mathbf{\Theta})\ge - \log \det (\mathbf{\Theta_{\mathcal{T}}}).
\]
Therefore, it remains to show that 
\begin{eqnarray*} \label{Equation:Thm1}
\begin{array}{rcl}
\langle \mathbf{\Theta, S}\rangle+\lambda_1 \|{\mathbf{Z}} - \text{diag}({\mathbf{Z}})\|_1 + \lambda_2 \|{\mathbf{V}} - \text{diag}({\mathbf{V}})\|_1 + \lambda_3  \|{\mathbf{V}} - \text{diag}({\mathbf{V}})\|_{1,q} &>& \\
\langle \mathbf{\Theta}_{\mathcal{T}}, \mathbf{S}\rangle+\lambda_1 \|{\mathbf{Z}_{\mathcal{T}}} - \text{diag}({\mathbf{Z}_{\mathcal{T}}})\|_1 + \lambda_2 \|{\mathbf{V}_{\mathcal{T}}} - \text{diag}({\mathbf{V}_{\mathcal{T}}})\|_1 + \lambda_3  \|{\mathbf{V}_{\mathcal{T}}} - \text{diag}({\mathbf{V}_{\mathcal{T}}})\|_{1,q},   \\
\end{array}
\end{eqnarray*}

\noindent or equivalently, that
\begin{equation*}
\langle \mathbf{\Theta}_{\mathcal{T}^c}, \mathbf{S} \rangle + \lambda_1 \| \mathbf{Z}_{\mathcal{T}^c}  \|_1+ \lambda_2 \| \mathbf{V}_{\mathcal{T}^c}\|_1+ \lambda_3(\|  \mathbf{V}-\text{diag}(\mathbf{V}) \|_{1,q} -\|  \mathbf{V}_{\mathcal{T}}-\text{diag}(\mathbf{V}_{\mathcal{T}}) \|_{1,q}  ) > 0.
\end{equation*}

\noindent Since $\|  \mathbf{V}-\text{diag}(\mathbf{V}) \|_{1,q} \ge \|  \mathbf{V}_{\mathcal{T}}-\text{diag}(\mathbf{V}_{\mathcal{T}}) \|_{1,q}$, it suffices to show that 
\begin{equation}
\label{Equation:Thm1-2-2}
\begin{split}
\langle \mathbf{\Theta}_{\mathcal{T}^c}, \mathbf{S} \rangle+ \lambda_1 \| \mathbf{Z}_{\mathcal{T}^c}  \|_1+ \lambda_2 \| \mathbf{V}_{\mathcal{T}^c}\|_1 > 0.
\end{split}
\end{equation}

\noindent Note that $\langle \mathbf{\Theta}_{\mathcal{T}^c}, \mathbf{S} \rangle$ =  $\langle \mathbf{\Theta}_{\mathcal{T}^c}, \mathbf{S}_{\mathcal{T}^c} \rangle$.  By the sufficient condition, $S_{\max} < \lambda_1$ and $2 S_{\max} < \lambda_2$.

\noindent In addition, we have that 
\begin{equation*}
\begin{split}
|\langle \mathbf{\Theta}_{\mathcal{T}^c}, \mathbf{S} \rangle|  &=|\langle \mathbf{\Theta}_{\mathcal{T}^c}, \mathbf{S}_{\mathcal{T}^c} \rangle| \\
&= |\langle \mathbf{V}_{\mathcal{T}^c}+\mathbf{V}^T_{\mathcal{T}^c}+\mathbf{Z}_{\mathcal{T}^c}, \mathbf{S}_{\mathcal{T}^c} \rangle| \\
&= |\langle 2\mathbf{V}_{\mathcal{T}^c}+\mathbf{Z}_{\mathcal{T}^c}, \mathbf{S}_{\mathcal{T}^c} \rangle| \\
&\le (2 \| \mathbf{V}_{\mathcal{T}^c} \|_1 +\|\mathbf{Z}_{\mathcal{T}^c}   \|_1 )S_{\max}\\
&<\ \lambda_2 \| \mathbf{V}_{\mathcal{T}^c} \|_1 +\lambda_1 \|\mathbf{Z}_{\mathcal{T}^c} \|_1,
\end{split}
\end{equation*}
where the last inequality follows from the sufficient condition.  We have shown (\ref{Equation:Thm1-2-2})  as desired.

\end{proof}

\subsection*{Proof of Theorem 2 (Necessary Condition)}
We first present a simple lemma for proving Theorem 2.  Throughout the proof of Theorem 2, $\| \cdot \|_\infty$ indicates the maximal absolute element of a matrix and $\|\cdot \|_{\infty,s}$ indicates the dual norm of $\| \cdot\|_{1,q}$.
\begin{lemma}
The dual representation of $\tilde{\mathbf{P}}( \mathbf{\Theta})$ in (\ref{normequation}) is
\begin{eqnarray}
\label{dualrepresentation}
\begin{array}{rccl}
\tilde{\mathbf{P}}^*(\mathbf{\Theta}) &=& \underset{\mathbf{X},\mathbf{Y},\mathbf{\Lambda}}\max & \langle \mathbf{\Lambda,\Theta} \rangle  \\
            && \mathrm{subject} \text{ } \mathrm{to} & \mathbf{\Lambda} + \mathbf{\Lambda}^T = \hat{\lambda}_2 \mathbf{X} + \hat{\lambda}_3 \mathbf{Y} \\
             &&			  & \|\mathbf{X}\|_{\infty} \leq 1, \|\mathbf{\Lambda}\|_{\infty} \leq 1, \|\mathbf{Y}\|_{\infty,s} \leq 1 \\
             &&			  & {X}_{ii} = 0, {Y}_{ii} = 0, {\Lambda}_{ii} = 0 \; \text{for } i=1,\ldots,p,
\end{array}
\end{eqnarray}
where $\frac{1}{s} + \frac{1}{q} = 1$.
\end{lemma}

\begin{proof}
We first state the dual representations for the norms in (\ref{normequation}):
\begin{eqnarray*}
\begin{array}{rccl}
\|\mathbf{Z}-\text{diag}(\mathbf{Z}) \|_1 &=& \underset{\mathbf{\Lambda}}\max & \langle \mathbf{\Lambda,Z} \rangle  \\            && \mbox{\text{subject to}} & \|\mathbf{\Lambda}  \|_{\infty} \le 1, \Lambda_{ii} = 0 \text{ for } i=1,\ldots, p,\\
\end{array}
\end{eqnarray*}
\begin{eqnarray*}
\begin{array}{rccl}
\|\mathbf{V}-\text{diag}(\mathbf{V}) \|_1 &=& \underset{\mathbf{X}}\max & \langle \mathbf{X,V} \rangle \\            && \mbox{\text{subject to}} & \|\mathbf{X}  \|_{\infty} \le 1, X_{ii} = 0 \text{ for } i=1,\ldots, p,\\
\end{array}
\end{eqnarray*}
\begin{eqnarray*}
\begin{array}{rccl}
\|\mathbf{V}-\text{diag}(\mathbf{V}) \|_{1,q} &=& \underset{\mathbf{Y}}\max & \langle \mathbf{Y,V} \rangle  \\ 
           && \mbox{\text{subject to}} & \|\mathbf{Y}  \|_{\infty,s} \le 1, Y_{ii} = 0 \text{ for } i=1,\ldots, p.\\
\end{array}
\end{eqnarray*}

\noindent Then, 
\begin{eqnarray*}
\begin{array}{rccl}
\tilde{ \mathbf{P}}(\mathbf{\Theta}) &=& \underset{{\mathbf{V, Z}}} \min &  \| \mathbf{Z}-\text{diag}(\mathbf{Z})\|_1 + \hat{\lambda}_2 \| \mathbf{V}-\text{diag}(\mathbf{V})\|_1 + \hat{\lambda}_3 \|\mathbf{V}-\text{diag}(\mathbf{V}) \|_{1,q}\\
 && \text{subject to} & \mathbf{\Theta} = \mathbf{Z+V+V}^T\\
 &=&  \underset{{\mathbf{V, Z}}} \min & \underset{\mathbf{\Lambda,X,Y}} \max  \langle \mathbf{\Lambda,Z} \rangle + \hat{\lambda}_2\langle \mathbf{X,V} \rangle + \hat{\lambda}_3\langle \mathbf{Y,V} \rangle \\
 && \mbox{\text{subject to}} & \|\mathbf{\Lambda}\|_{\infty} \le 1, \|\mathbf{X}\|_{\infty} \le 1, \|\mathbf{Y}\|_{\infty,s}\le 1 \\
             &&			  & \Lambda_{ii}=0, X_{ii}=0, Y_{ii}=0 \text{ for } i=1,\ldots, p   \\
             &&			  & \mathbf{\Theta} = \mathbf{Z+V+V}^T\\
 &=& \underset{\mathbf{\Lambda,X,Y}} \max &  \underset{{\mathbf{V, Z}}} \min   \langle \mathbf{\Lambda,Z} \rangle + \hat{\lambda}_2 \langle \mathbf{X,V} \rangle + \hat{\lambda}_3 \langle \mathbf{Y,V} \rangle \\
 && \mbox{\text{subject to}} & \|\mathbf{\Lambda}\|_{\infty} \le 1, \|\mathbf{X}\|_{\infty} \le 1, \|\mathbf{Y}\|_{\infty,s}\le 1 \\
             &&			  & \Lambda_{ii}=0, X_{ii}=0, Y_{ii}=0 \text{ for } i=1,\ldots, p   \\
             &&			  & \mathbf{\Theta} = \mathbf{Z+V+V}^T\\
             
&=&\underset{\mathbf{\Lambda},\mathbf{X},\mathbf{Y}}\max & \langle \mathbf{\Lambda,\Theta} \rangle \\            && \mbox{\text{subject to}} & \mathbf{\Lambda} + \mathbf{\Lambda}^T = \hat{\lambda}_2 \mathbf{X} + \hat{\lambda}_3 \mathbf{Y} \\
             &&			  & \|\mathbf{X}\|_{\infty} \leq 1, \|\mathbf{\Lambda}\|_{\infty} \leq 1, \|\mathbf{Y}\|_{\infty,s} \leq 1 \\
             &&			  & {X}_{ii} = 0, {Y}_{ii} = 0, {\Lambda}_{ii} = 0 \; \text{for } i=1,\ldots,p.
\end{array}
\end{eqnarray*}

\noindent The third equality holds since the constraints on $(\mathbf{V,Z})$ and on $(\mathbf{\Lambda,X,Y})$ are both compact convex sets and so by the minimax theorem, we can swap max and min.  The last equality follows from the fact that 
\begin{eqnarray*}
\begin{array}{ccc}
\underset{\mathbf{V},\mathbf{Z}}\min & \langle \mathbf{\Lambda,Z} \rangle + \hat{\lambda}_2  \langle \mathbf{X,V} \rangle + \hat{\lambda}_3  \langle \mathbf{Y,V} \rangle \\
\mbox{subject to} & \mathbf{\Theta} = \mathbf{Z} + \mathbf{V} + \mathbf{V}^T \\
=&
\left\{ \begin{array}{cc} \langle \mathbf{\Lambda,\Theta} \rangle& \mbox{if  } \mathbf{\Lambda} + \mathbf{\Lambda}^T = \hat{\lambda}_2 \mathbf{X} + \hat{\lambda}_3 \mathbf{Y} \\
			       -\infty & \mbox{otherwise}. \end{array} \right.
\end{array}
\end{eqnarray*}
\end{proof}

We now present the proof of Theorem 2.

\begin{proof}
The optimality condition for (\ref{Equation:reduceHGL}) is given by
\begin{equation}
\label{Equation:optimalitycondition}
\mathbf{0} = -\mathbf{\Theta}^{-1} + \mathbf{S} + \lambda_1\mathbf{\Lambda},
\end{equation}
where $\mathbf{\Lambda}$ is a subgradient of $\tilde{\mathbf{P}}(\mathbf{\Theta})$ in (\ref{normequation}) and the left-hand side of the above equation is a zero matrix of size $p\times p$.

Now suppose that $\mathbf{\Theta}^*$ that solves (\ref{Equation:optimalitycondition}) is supported on $\mathcal{T}$, i.e., $\mathbf{\Theta}^*_{\mathcal{T}^c} = 0$.  Then for any $(i,j) \in \mathcal{T}^c$, we have that 
\begin{equation}
\label{Equation:optimconditionhold}
0 = S_{ij} + \lambda_1 {\Lambda}^*_{ij},
\end{equation}
where $\mathbf{\Lambda}^*$ is a subgradient of $\tilde{\mathbf{P}}(\mathbf{\Theta}^*)$.  Note that $\mathbf{\Lambda}^*$ must be an optimal solution to the optimization problem (\ref{dualrepresentation}).  Therefore, it is also a feasible solution to (\ref{dualrepresentation}), implying that 
\begin{equation*}
\begin{split}
|\Lambda_{ij}^*+\Lambda_{ji}^*| &\le \hat{\lambda}_2 + \hat{\lambda}_3,\\
|\Lambda_{ij}^*| &\le 1.
\end{split}
\end{equation*}
From (\ref{Equation:optimconditionhold}), we have that $\Lambda_{ij}^* = -\frac{S_{ij}}{\lambda_1}$ and thus,
\begin{equation*}
\begin{split}
\lambda_1 &\ge \lambda_1 \underset{(i,j)\in \mathcal{T}^c}\max |\Lambda_{ij}^*| \\
&=   \lambda_1 \underset{(i,j)\in \mathcal{T}^c}\max \frac{|S_{ij}|}{\lambda_1}\\
&= S_{\max}.
\end{split}
\end{equation*}
Also, recall that $\hat{\lambda}_2 = \frac{\lambda_2}{\lambda_1}$ and $\hat{\lambda}_3 = \frac{\lambda_3}{\lambda_1}$. We have that 
\begin{equation*}
\begin{split}
\lambda_2+\lambda_3 &\ge \lambda_1 \underset{(i,j)\in \mathcal{T}^c}\max |\Lambda_{ij}^*+\Lambda_{ji}^*| \\
&=   \lambda_1 \underset{(i,j)\in \mathcal{T}^c}\max \frac{2|S_{ij}|}{\lambda_1}\\
&= 2S_{\max}.
\end{split}
\end{equation*}
Hence, we obtain the desired result.

\end{proof}


\section*{Appendix C: Some Properties of HGL}
\subsection*{Proof of Lemma \ref{Lemma:DiagonalZ}}

\begin{proof} Let $(\mathbf{\Theta}^*,\mathbf{Z}^*,\mathbf{V}^*)$ be the solution to (\ref{Eq:ggmhub}) and suppose that $\mathbf{Z}^*$ is not a diagonal matrix.  Note that  $\mathbf{Z}^*$ is symmetric since $\mathbf{\Theta} \in \mathcal{S} \equiv \{\mathbf{\Theta}:\mathbf{\Theta}\succ 0 \text{ and } \mathbf{\Theta}= \mathbf{\Theta}^T \}$.
Let $\hat{\mathbf{Z}}=\text{diag}(\mathbf{Z}^*)$,  a matrix that contains the diagonal elements of the matrix $\mathbf{Z}^*$. Also, construct $\hat{\mathbf{V}}$ as follows, 
\begin{eqnarray*}
\hat{\mathbf{V}}_{ij} = \left\{\begin{array}{cc} \mathbf{V}^*_{ij} + \frac{\mathbf{Z}^*_{ij}}{2} & \mbox{if } i \neq j \\
				          \mathbf{V}^*_{jj} & \mbox{otherwise}. \end{array} \right.
\end{eqnarray*}			

\noindent Then, we have that $\mathbf{\Theta}^* = \hat{\mathbf{Z}} + \hat{\mathbf{V}} + \hat{\mathbf{V}}^T$. Thus, $(\mathbf{\Theta}^*,\hat{\mathbf{Z}},\hat{\mathbf{V}})$ is a feasible solution to (\ref{Eq:ggmhub}).  We now show that $(\mathbf{\Theta}^*,\hat{\mathbf{Z}},\hat{\mathbf{V}})$ has a smaller objective than $(\mathbf{\Theta}^*,\mathbf{Z}^*,\mathbf{V}^*)$ in (\ref{Eq:ggmhub}), giving us a contradiction.  Note that

\begin{eqnarray*} \label{eq:bound-1}
\begin{array}{rcl}
\lambda_1 \|\hat{\mathbf{Z}} - \text{diag}(\hat{\mathbf{Z}})\|_1 + \lambda_2 \|\hat{\mathbf{V}} - \text{diag}(\hat{\mathbf{V}})\|_1 &=& \lambda_2 \|\hat{\mathbf{V}} - \text{diag}(\hat{\mathbf{V}})\|_1 \\
											  &=& \lambda_2 \sum_{i \neq j} |\mathbf{V}^*_{ij} + \frac{\mathbf{Z}^*_{ij}}{2}| \\
								                          &\leq& \lambda_2 \|\mathbf{V}^* - \text{diag}(\mathbf{V}^*)\|_1 + \frac{\lambda_2}{2}\|\mathbf{Z}^* - \text{diag}(\mathbf{Z}^*)\|_1,
\end{array}								                         
\end{eqnarray*}
and
\begin{eqnarray*} \label{eq:bound-2}
\begin{array}{rcl}
\lambda_3 \sum_{j=1}^p \|(\hat{\mathbf{V}} - \text{diag}(\hat{\mathbf{V}}))_j\|_q &\leq& \lambda_3 \sum_{j=1}^p \|(\mathbf{V}^* - \text{diag}(\mathbf{V}^*))_j\|_q + \frac{\lambda_3}{2} \sum_{j=1}^p \|(\mathbf{Z}^* - \text{diag}(\mathbf{Z}^*))_j\|_{q} \\
							      &\leq& \lambda_3 \sum_{j=1}^p \|(\mathbf{V}^* - \text{diag}(\mathbf{V}^*))_j\|_q + \frac{\lambda_3}{2} \|\mathbf{Z}^* - \text{diag}(\mathbf{Z}^*)\|_1,
\end{array}
\end{eqnarray*}
where the last inequality follows from the fact that for any vector $\mathbf{x} \in \mathbb{R}^p$ and $q\ge1$, $\|\mathbf{x}\|_q$ is a nonincreasing function of $q$ \citep{Gentle}.

Summing up the above inequalities, we get that 
\begin{eqnarray*}
\begin{array}{rcl}
\lambda_1 \|\hat{\mathbf{Z}} - \text{diag}(\hat{\mathbf{Z}})\|_1 + \lambda_2 \|\hat{\mathbf{V}} - \text{diag}(\hat{\mathbf{V}})\|_1 + \lambda_3 \sum_{j=1}^p \|(\hat{\mathbf{V}} - \text{diag}(\hat{\mathbf{V}}))_j\|_q &\leq& \\
\frac{\lambda_2 + \lambda_3}{2}\|\mathbf{Z}^* - \text{diag}(\mathbf{Z}^*)\|_1 + \lambda_2 \|\mathbf{V}^* - \text{diag}(\mathbf{V}^*)\|_1 + \lambda_3 \sum_{j=1}^p \|(\mathbf{V}^* - \text{diag}(\mathbf{V}^*))_j\|_q   &<& \\
\lambda_1\|\mathbf{Z}^* - \text{diag}(\mathbf{Z}^*)\|_1 + \lambda_2 \|\mathbf{V}^* - \text{diag}(\mathbf{V}^*)\|_1 + \lambda_3 \sum_{j=1}^p \|(\mathbf{V}^* - \text{diag}(\mathbf{V}^*))_j\|_q,
\end{array}
\end{eqnarray*}
where the last inequality uses the assumption that $\lambda_1 > \frac{\lambda_2+\lambda_3}{2}$.  We arrive at a contradiction and therefore the result holds.
\end{proof}


\subsection*{Proof of Lemma \ref{Lemma:DiagonalV}}

\begin{proof} Let $(\mathbf{\Theta}^*,\mathbf{Z}^*,\mathbf{V}^*)$ be the solution to (\ref{Eq:ggmhub}) and suppose $\mathbf{V}^*$ is not a diagonal matrix.
Let $\hat{\mathbf{V}}=\text{diag}(\mathbf{V}^*)$, a diagonal matrix that contains the diagonal elements of $\mathbf{V}^*$. Also construct $\hat{\mathbf{Z}}$ as follows,
\begin{eqnarray*}
\hat{\mathbf{Z}}_{ij} = \left\{\begin{array}{rc} \mathbf{Z}^*_{ij} + \mathbf{V}^*_{ij} + \mathbf{V}^*_{ji} & \mbox{if } i \neq j \\
					  \mathbf{Z}^*_{ij} & \mbox{otherwise}. \end{array} \right.
\end{eqnarray*}
Then, we have that $\mathbf{\Theta}^* = \hat{\mathbf{V}} + \hat{\mathbf{V}}^T + \hat{\mathbf{Z}}$. We now show that $(\mathbf{\Theta}^*,\hat{\mathbf{Z}}, \hat{\mathbf{V}})$ has a smaller objective value
than $(\mathbf{\Theta}^*,\mathbf{Z}^*,\mathbf{V}^*)$ in (\ref{Eq:ggmhub}), giving us a contradiction.  We start by noting that 
\begin{eqnarray*} \label{eq:bound-4}
\begin{array}{rcl}
\lambda_1 \|\hat{\mathbf{Z}} - \text{diag}(\hat{\mathbf{Z}})\|_1 + \lambda_2 \|\hat{\mathbf{V}} - \text{diag}(\hat{\mathbf{V}})\|_1 &=& \lambda_1 \|\hat{\mathbf{Z}} - \text{diag}(\hat{\mathbf{Z}})\|_1 \\
											    &\leq& \lambda_1 \|\mathbf{Z}^* - \text{diag}(\mathbf{Z}^*)\|_1 + 2\lambda_1 \|\mathbf{V}^* - \text{diag}(\mathbf{V}^*)\|_1.
\end{array}
\end{eqnarray*}
By Holder's Inequality, we know that $\mathbf{x}^T \mathbf{y} \le \|\mathbf{x}\|_q \| \mathbf{y}\|_s$ where $\frac{1}{s}+\frac{1}{q}=1$ and $\mathbf{x,y} \in \mathbb{R}^{p-1}$.  Setting $\mathbf{y} = \text{sign}(\mathbf{x})$, we have that $\|\mathbf{x}\|_1 \le (p-1)^{\frac{1}{s}} \|\mathbf{x}\|_q$.  Consequently,
\begin{equation*}
 \frac{\lambda_3}{(p-1)^{\frac{1}{s}}} \|\mathbf{V}^* - \text{diag}(\mathbf{V}^*) \|_1 \le \lambda_3 \sum_{j=1}^p  \| (  \mathbf{V}^* - \text{diag}(\mathbf{V}^*)    )_j \|_q. 
\end{equation*}
Combining these results, we have that
\begin{equation*}
\begin{split}
&\lambda_1 \|\hat{\mathbf{Z}} - \text{diag}(\hat{\mathbf{Z}})\|_1 + \lambda_2 \|\hat{\mathbf{V}} - \text{diag}(\hat{\mathbf{V}})\|_1+ \lambda_3 \sum_{j=1}^p  \| (  \hat{\mathbf{V}} - \text{diag}(\hat{\mathbf{V}})    )_j \|_q \\
&\le \lambda_1 \|\mathbf{Z}^* - \text{diag}(\mathbf{Z}^*)\|_1 + 2\lambda_1 \|\mathbf{V}^* - \text{diag}(\mathbf{V}^*)\|_1\\
&< \lambda_1 \|\mathbf{Z}^* - \text{diag}(\mathbf{Z}^*)\|_1 + \left( \lambda_2+\frac{\lambda_3}{(p-1)^{\frac{1}{s}}} \right) \|\mathbf{V}^* - \text{diag}(\mathbf{V}^*)\|_1\\
&\le \lambda_1 \|\mathbf{Z}^* - \text{diag}(\mathbf{Z}^*)\|_1 + \lambda_2 \|\mathbf{V}^* - \text{diag}(\mathbf{V}^*)\|_1 + \lambda_3 \sum_{j=1}^p \|(\mathbf{V}^* - \text{diag}(\mathbf{V}^*))_j \|_q,
\end{split}
\end{equation*}
where we use the assumption that $\lambda_1 < \frac{\lambda_2}{2}+\frac{\lambda_3}{2(p-1)^{\frac{1}{s}}}$.  This leads to a contradiction.
\end{proof}


\subsection*{Proof of Lemma \ref{lemma3}}
In this proof, we consider the case when $\lambda_1 > \frac{\lambda_2+\lambda_3}{2}$. A similar proof technique can be used to prove the case when  $\lambda_1 < \frac{\lambda_2+\lambda_3}{2}$. \\

\begin{proof}  Let $f({\bf \Theta}, {\bf V}, {\bf Z})$ denote the objective of (\ref{Eq:ggmhub}) with $q=1$, and $(\mathbf{\Theta}^*, \mathbf{V}^*, \mathbf{Z}^*)$ the optimal solution.  By Lemma~\ref{Lemma:DiagonalZ}, the assumption that
$\lambda_1 > \frac{\lambda_2+\lambda_3}{2}$ implies that 
   $\mathbf{Z}^*$ is a diagonal matrix. Now let $\hat{\bf V} = \frac{1}{2} \left( {\bf V}^* + ({\bf V}^*)^T \right)$. Then 
    
\begin{eqnarray*}
f(\mathbf{\Theta}^*, \hat{\mathbf{V}}, \mathbf{Z}^*) 
&=& -\log \det \mathbf{\Theta}^* + \langle \mathbf{\Theta}^*,\mathbf{S} \rangle + \lambda_1 \| {\mathbf{Z}}^*-\text{diag}({\mathbf{Z}}^*)  \|_1 + (\lambda_2+\lambda_3) \| \hat{\mathbf{V}}-\text{diag}(\hat{\mathbf{V}}) \|_1  \nonumber     \\
&=&-\log \det \mathbf{\Theta}^* + \langle \mathbf{\Theta}^*,\mathbf{S} \rangle + \frac{\lambda_2+\lambda_3}{2} \| {\mathbf{V}}^*+ {{\mathbf{V}}^*}^T-\text{diag}({\mathbf{V}}^*+{{\mathbf{V}}^*}^T) \|_1  \nonumber \\
& \leq& -\log \det \mathbf{\Theta}^* + \langle \mathbf{\Theta}^*,\mathbf{S} \rangle + ({\lambda_2+\lambda_3}) \| {\mathbf{V}}^*-\text{diag}({\mathbf{V}}^*) \|_1 \nonumber \\
& = &  f(\mathbf{\Theta}^*, {\mathbf{V}}^*, \mathbf{Z}^*) \nonumber \\ 
& \leq & f(\mathbf{\Theta}^*, \hat{\mathbf{V}}, \mathbf{Z}^*),
\end{eqnarray*}
where the last inequality follows from the assumption that ($\mathbf{\Theta}^*, {\mathbf{V}}^*, \mathbf{Z}^*$) solves (\ref{Eq:ggmhub}).  By strict convexity of $f$, this means that ${\bf V}^* = \hat{\bf V}$, i.e., $\mathbf{V}^*$ is symmetric. 
This implies that 
\begin{eqnarray} 
\label{my.eq}
f(\mathbf{\Theta}^*, {\mathbf{V}}^*, \mathbf{Z}^*) 
&=&-\log \det \mathbf{\Theta}^* + \langle \mathbf{\Theta}^*,\mathbf{S} \rangle + \frac{\lambda_2+\lambda_3}{2} \| {\mathbf{V}}^*+ {{\mathbf{V}}^*}^T-\text{diag}({\mathbf{V}}^*+{{\mathbf{V}}^*}^T) \|_1  \nonumber \\
&=&-\log \det \mathbf{\Theta}^* +\langle \mathbf{\Theta}^*,\mathbf{S} \rangle +\frac{\lambda_2+\lambda_3}{2} \| \mathbf{\Theta}^* - \text{diag}(\mathbf{\Theta}^*)\|_1 \\
&=& g({\bf \Theta}^*),\nonumber
\end{eqnarray}
where $g({\bf \Theta})$ is the objective of the  graphical lasso optimization problem, evaluated at $\bf \Theta$, with tuning parameter $\frac{\lambda_2+\lambda_3}{2}$.
Suppose that $\tilde{\bf \Theta}$ minimizes $g({\bf \Theta})$, and ${\bf \Theta}^* \neq \tilde{\bf \Theta}$.  
Then, by (\ref{my.eq}) and strict convexity of $g$, $g(\mathbf{\Theta}^*) = f(\mathbf{\Theta}^*,{\mathbf{V}}^*,{\mathbf{Z}}^*) \le f(\tilde{\mathbf{\Theta}},\tilde{\mathbf{\Theta}}/2,\mathbf{0}) = g(\tilde{\mathbf{\Theta}}) < g(\mathbf{\Theta}^*)$, giving us a contradiction.  Thus it must be that $\tilde{\mathbf{\Theta}} = \mathbf{\Theta}^*$.

\end{proof}

\section*{Appendix D: Simulation Study for Hub Covariance Graph }
In this section, we present the results for the simulation study described in Section~\ref{Cov:simulation} with $n=100$, $p=200$, and $|\mathcal{H}| = 4$. We calculate the proportion of correctly estimated hub nodes with $r=40$.  The results are shown in Figure~\ref{Fig:CovSimsmall}.  As we can see from Figure~\ref{Fig:CovSimsmall}, our proposal outperforms \citet{BienTibs11}.  In particular, we can see from Figure~\ref{Fig:CovSimsmall}(c) that \citet{BienTibs11} fails to identify hub nodes.  
\begin{figure}[htp]
\begin{center}
\includegraphics[scale=0.51]{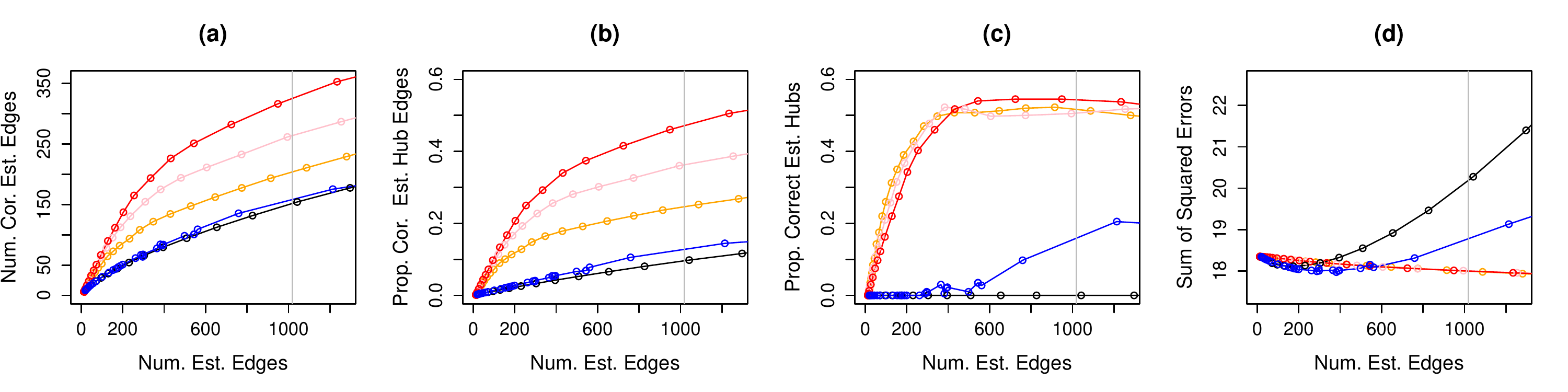}
 \end{center}
  \caption{Covariance graph simulation with $n=100$ and $p=200$.  Details of the axis labels are as in Figure~\ref{Fig:simulation1}.  The colored lines correspond to the proposal of {\protect\citet{Xueetal2012}} (\protect\includegraphics[height=0.5em]{black.png}); HCG with $\lambda_3=1$ (\protect\includegraphics[height=0.5em]{orange.png}), $\lambda_3=1.5$ (\protect\includegraphics[height=0.5em]{pink.png}), and  $\lambda_3=2$ (\protect\includegraphics[height=0.5em]{red.png});  and the proposal of {\protect\citet{BienTibs11}} (\protect\includegraphics[height=0.5em]{blue.png}). }
  \label{Fig:CovSimsmall}
\end{figure}

\section*{Appendix E: Run Time Study for the ADMM algorithm for HGL}
In this section, we present a more extensive run time study for the ADMM algorithm for HGL.  We ran experiments with $p=100,200,300$ and with $n=p/2$ on a 2.26GHz Intel Core 2 Duo machine.  Results averaged over 10 replications are displayed in Figures~\ref{fig:timing}(a)-(b), where the panels depict the run time and number of iterations required for the algorithm to converge, as a function of $\lambda_1$, with $\lambda_2=0.5$ and $\lambda_3=2$ fixed.  The number of iterations required for the algorithm to converge is computed as the total number of iterations in Step 2 of Algorithm~\ref{Alg:general}.  We see from Figure~\ref{fig:timing}(a) that as $p$ increases from 100 to 300, the run times increase substantially, but never exceed several minutes.  Note that these results are without using the block diagonal condition in Theorem 1.

\begin{figure}[htp]
\begin{center}
\hspace{10mm} (a) \hspace{70mm} (b)
\includegraphics[scale=0.36]{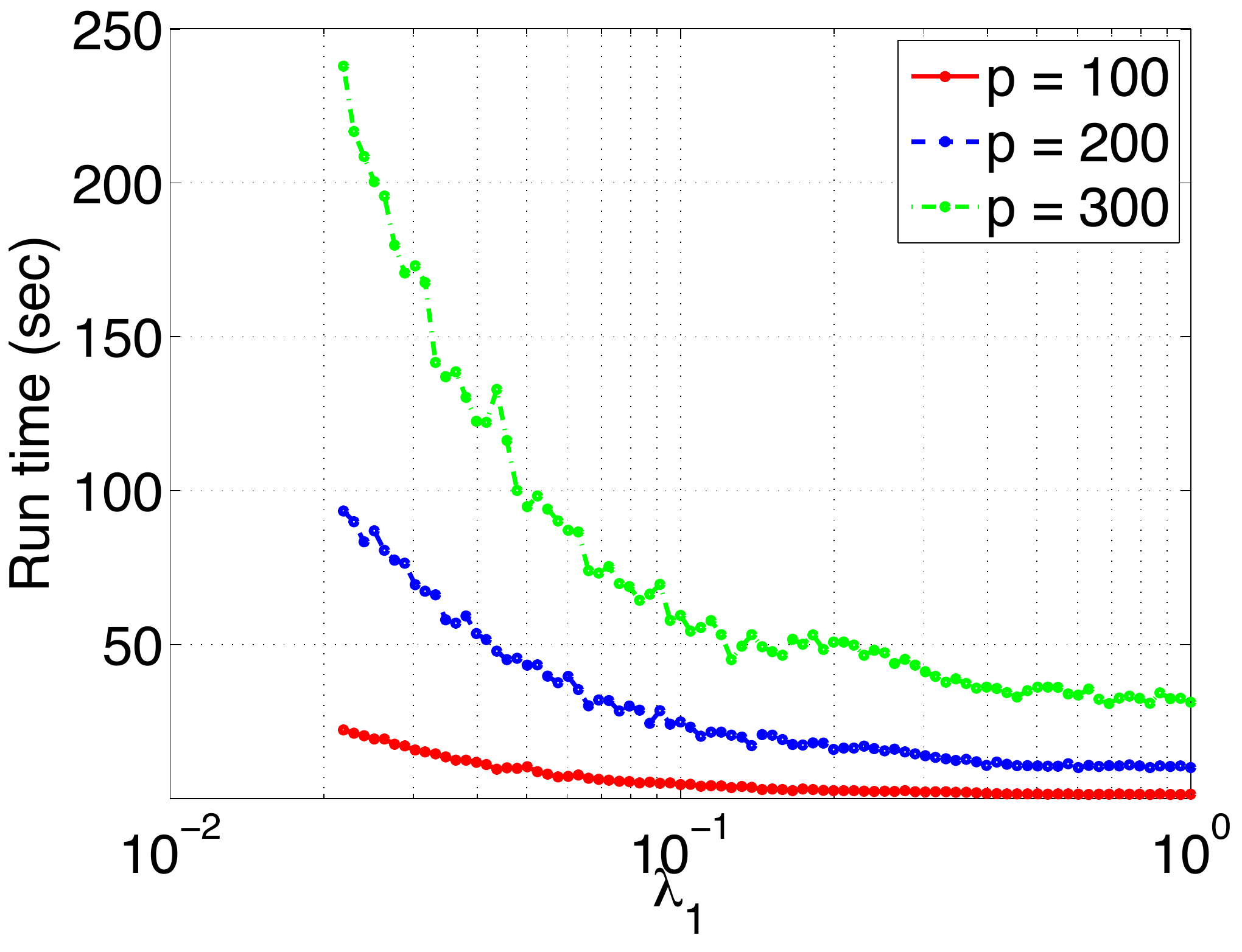}
\includegraphics[scale=0.36]{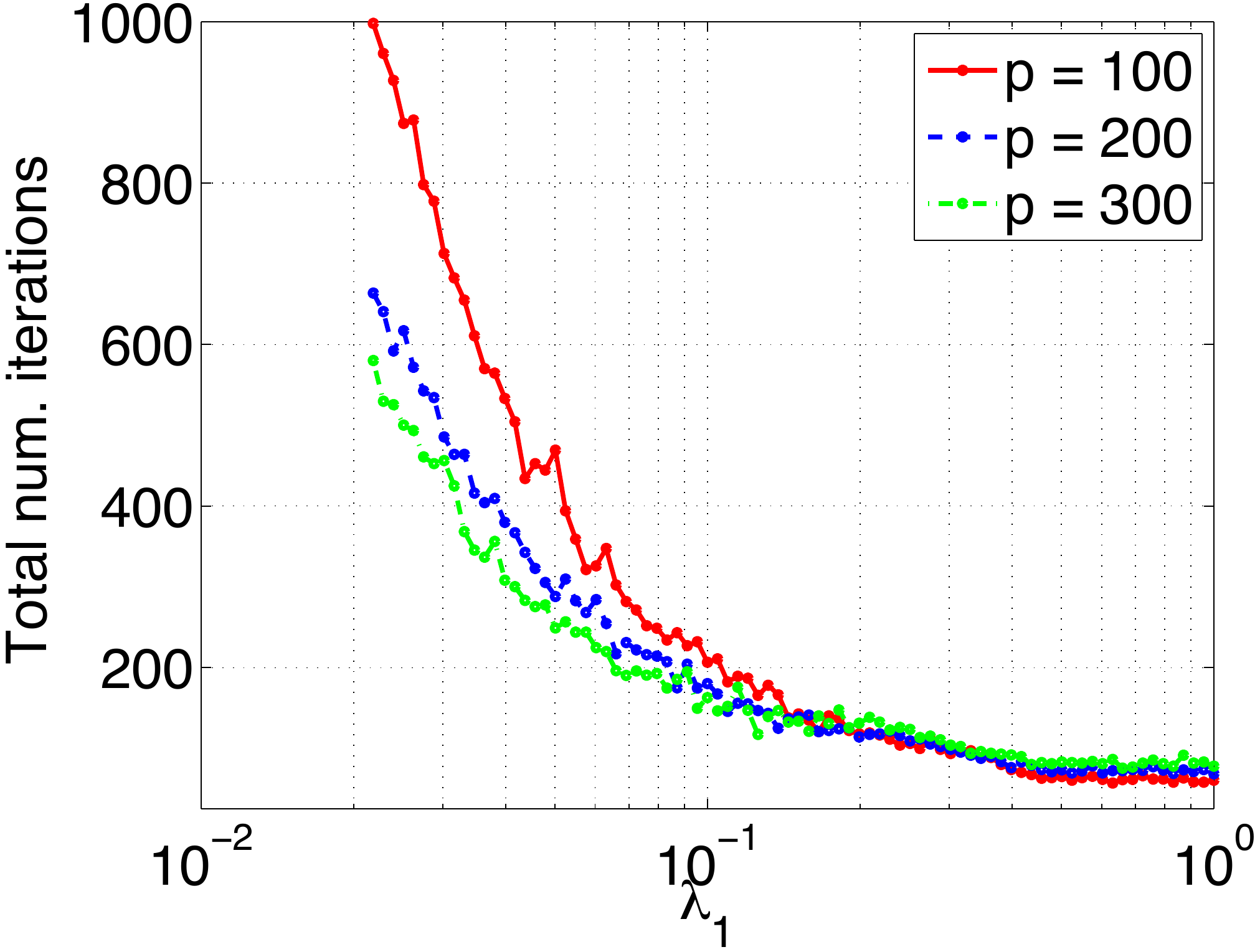}
\end{center}
\caption{(a): Run time (in seconds) of the ADMM algorithm for HGL, as a function of $\lambda_1$, for fixed values of $\lambda_2$ and $\lambda_3$.  (b): The total number of iterations required for the ADMM algorithm for HGL to converge, as a function of $\lambda_1$.  All results are averaged over 10 simulated data sets. These results are without using the block diagonal condition in Theorem 1.}
\label{fig:timing}
\end{figure}

\section*{Appendix F: Update for $\mathbf{\Theta}$ in Step 2(a)i for Binary Ising Model using Barzilai-Borwein Method}
We consider updating $\mathbf{\Theta}$ in Step 2(a)i of Algorithm~\ref{Alg:general} for binary Ising  model. 
 Let 
\[
h(\mathbf{\Theta}) = \left\{  -\sum_{j=1}^p \sum_{j'=1}^p \theta_{jj'}  (\mathbf{X}^T\mathbf{X})_{jj'} +\sum_{i=1}^p \sum_{j=1}^p \log \left(1+ \exp\left[ \theta_{jj}+\sum_{j'\ne j} \theta_{jj'} x_{ij'} \right] \right) +\frac{\rho}{2} \|\mathbf{\Theta}-\tilde{\mathbf{\Theta}}+\mathbf{W}_1\|_F^2   \right\}.
\]
Then, the optimization problem for Step 2(a)i of Algorithm~\ref{Alg:general} is 
\begin{equation}
\label{Eq:isingalg}
\underset{\mathbf{\Theta} \in \mathcal{S}}{\text{minimize}} \quad h(\mathbf{\Theta}),
\end{equation}
where  $\mathcal{S}=\{\mathbf{\Theta}: \mathbf{\Theta}=\mathbf{\Theta}^T\}$.  In solving (\ref{Eq:isingalg}), we will treat $\mathbf{\Theta} \in \mathcal{S}$ as an implicit constraint.

The Barzilai-Borwein method is a gradient descent method with the step-size chosen to mimic the secant condition of the BFGS method \citep[see, e.g.,][]{barzilai1988two,nocedal2006numerical}.  The convergence of the Barzilai-Borwein method for unconstrained minimization using a non-monotone line search was shown in \citet{raydan1997barzilai}.  Recent convergence results for  a quadratic cost function can be found in \citet{dai2013new}. To implement the Barzilai-Borwein method, we need to evaluate the gradient of $h(\mathbf{\Theta})$.  Let $\nabla h(\mathbf{\Theta})$ be a $p\times p$ matrix, where the $(j,j')$ entry is the gradient of $h(\mathbf{\Theta})$ with respect to $\theta_{jj'}$, computed under the constraint $\mathbf{\Theta} \in \mathcal{S}$, that is,  $\theta_{jj'}=\theta_{j'j}$. Then, 
\[
(\nabla  h(\mathbf{\Theta}))_{jj} =-(\mathbf{X}^T\mathbf{X})_{jj} + \sum_{i=1}^n \left[ \frac{\exp(\theta_{jj} + \sum_{j'\ne j} \theta_{jj'} x_{ij'})}{1+\exp(\theta_{jj} + \sum_{j'\ne j} \theta_{jj'} x_{ij'})}\right] + \rho (\theta_{jj} - \tilde{\theta}_{jj} + (\mathbf{W}_1)_{jj}),  
\]
and 
\begin{equation*}
\begin{split}
(\nabla h(\mathbf{\Theta}))_{jj'} &= -2(\mathbf{X}^T\mathbf{X})_{jj} + 2\rho  (\theta_{jj'} - \tilde{\theta}_{jj'} + (\mathbf{W}_1)_{jj'})  \\
&\qquad + \sum_{i=1}^n \left[ \frac{ x_{ij'}\exp(\theta_{jj} + \sum_{j'\ne j} \theta_{jj'} x_{ij'})}{1+\exp(\theta_{jj} + \sum_{j'\ne j} \theta_{jj'} x_{ij'})} + \frac{ x_{ij}\exp(\theta_{j'j'} + \sum_{j\ne j'} \theta_{jj'} x_{ij})}{1+\exp(\theta_{j'j'} + \sum_{j\ne j'} \theta_{jj'} x_{ij})} \right].
\end{split}
\end{equation*}

A simple implementation of the Barzilai-Borwein algorithm for solving (\ref{Eq:isingalg}) is detailed in Algorithm~\ref{Alg:bb}.  We note that the Barzilai-Borwein algorithm can be improved \citep[see, e.g.,][]{barzilai1988two,wright2009sparse}. We leave such improvement for future work.

\begin{algorithm}[htp]
\small
\caption{Barzilai-Borwein Algorithm for Solving (\ref{Eq:isingalg}).}
\label{Alg:bb}

\begin{enumerate}
\item  \textbf{Initialize} the parameters: 
\begin{enumerate}
\item $\mathbf{\Theta}_1= \mathbf{I}$ and $\mathbf{\Theta}_{0}= 2\mathbf{I}$. 
\item constant $\tau>0$.
\end{enumerate}
 
\item  \textbf{Iterate} until the stopping criterion $\frac{\| {\mathbf{\Theta}}_t- {\mathbf{\Theta}}_{t-1} \|_F^2}{\| {\mathbf{\Theta}}_{t-1}\|_F^2} \le \tau$ is met, where $\mathbf{\Theta}_t$ is the value of $\mathbf{\Theta}$ obtained at the $t$th iteration:
\begin{enumerate}
\item $\alpha_t= \text{trace}\left[ (\mathbf{\Theta}_t-\mathbf{\Theta}_{t-1})^T (\mathbf{\Theta}_t-\mathbf{\Theta}_{t-1})\right] / \text{trace}\left[ (\mathbf{\Theta}_t-\mathbf{\Theta}_{t-1})^T (\nabla h(\mathbf{\Theta}_t) - \nabla h(\mathbf{\Theta}_{t-1}))\right]$.
\item $\mathbf{\Theta}_{t+1} = \mathbf{\Theta}_t - \alpha_t \nabla h(\mathbf{\Theta}_t)$.
\end{enumerate}

\end{enumerate}
\end{algorithm}

\newpage
\bibliographystyle{natbib}
\bibliography{reference}

\end{document}